\documentclass{article}

\usepackage[nonatbib,preprint]{neurips_2022}

\usepackage[numbers]{natbib}
\usepackage[utf8]{inputenc} %
\usepackage[T1]{fontenc}    %
\usepackage{hyperref}       %
\usepackage{url}            %
\usepackage{nicefrac}

\usepackage{graphicx}
\usepackage{subcaption}
\usepackage{booktabs} %
\usepackage[dvipsnames]{xcolor}
\usepackage{amsmath}
\usepackage{amsthm}
\usepackage{bm}
\usepackage{tikz}
\usetikzlibrary{fit,arrows.meta,external,math,decorations.pathmorphing}
\usepackage{amsmath,amsfonts,amssymb}
\usepackage{xcolor}

\usepackage{wrapfig}
\usepackage{enumitem}

\usetikzlibrary{matrix}
\definecolor{ferngreen}{rgb}{0.31, 0.47, 0.26}
\definecolor{denim}{rgb}{0.08, 0.38, 0.74}
\definecolor{deepsaffron}{rgb}{1.0, 0.6, 0.2}
\definecolor{llgrey}{rgb}{0.95, 0.95, 0.95}
\definecolor{llyellow}{rgb}{0.98, 0.98, 0.75}
\definecolor{llred}{rgb}{0.97, 0.85, 0.85}

\usepackage{mathrsfs}

\newtheorem{lemma}{Lemma}
\newtheorem{proposition}{Proposition}
\newtheorem{theorem}{Theorem}

\usepackage{enumitem}
\setlist[description]{leftmargin=*}

\usepackage{hyperref}

\usepackage{amssymb}
\usepackage{amsmath}
\usepackage{bbm}

\newcommand{\Pa}[1]{\text{pa}(#1)} %
\newcommand{\Do}[1]{\text{do}(#1)} %

\newcommand{\E}{\mathbb{E}}

\def\rva{{\mathbf{a}}}
\def\rvb{{\mathbf{b}}}
\def\rvc{{\mathbf{c}}}

\def\rvx{{\mathbf{x}}}

\def\rvz{{\mathbf{z}}}

\def\rmX{{\mathbf{X}}}

\usepackage{cleveref}
\usepackage{layouts}

\newtheorem{assumption}{Assumption}

\definecolor{jazzberryjam}{rgb}{0.65, 0.04, 0.37}

\newcommand\notears{\textit{Notears}}
\newcommand\notearsmlp{\textit{Notears-MLP}}
\newcommand\grandag{\textit{Grandag}}
\newcommand\carefl{\textit{Carefl}}
\newcommand\notearssob{\textit{Notears-Sob}}
\newcommand\golem{\textit{Golem}}
\newcommand\wt{W(\theta)}

\title{Deep End-to-end Causal Inference}

\newcommand{\spacer}{\ \ }
\author{%
  Tomas Geffner\textsuperscript{\normalfont $\dagger$ 1 }\thanks{Equal contribution. $\dagger$ Contributed during internship or residency in Microsoft Research. } \spacer Javier Antoran\textsuperscript{\normalfont $\dagger$ 2 *} \spacer   Adam Foster\textsuperscript{\normalfont 3 *} \spacer     Wenbo Gong\textsuperscript{\normalfont 3} \spacer  Chao Ma\textsuperscript{\normalfont 3} \\ \textbf{Emre Kiciman}\textsuperscript{\normalfont 3} \spacer \textbf{Amit Sharma}\textsuperscript{\normalfont 3} \spacer \textbf{Angus Lamb}\textsuperscript{\normalfont $\dagger$ 4} \spacer \textbf{Martin Kukla}\textsuperscript{\normalfont 3} \spacer  \\ \textbf{Nick Pawlowski}\textsuperscript{\normalfont 3} \spacer \textbf{Miltiadis Allamanis}\textsuperscript{\normalfont 3} \spacer \textbf{Cheng Zhang}\textsuperscript{\normalfont 3} \\
  \textsuperscript{1} University of Massachusetts Amherst \qquad
  \textsuperscript{2} University of Cambridge \\
  \textsuperscript{3} Microsoft Research \qquad
  \textsuperscript{4} G-Research \\
  \texttt{cheng.zhang@microsoft.com}
}

\begin{document}

\maketitle

\begin{abstract}
Causal inference is essential for data-driven decision making across domains such as business engagement, medical treatment and policy making.  However, research on causal discovery has evolved separately from inference methods, preventing straight-forward combination of methods from both fields.  In this work, we develop Deep End-to-end Causal Inference (DECI), a single flow-based non-linear additive noise model that takes in observational data and can perform both causal discovery and inference, including conditional average treatment effect (CATE) estimation. We provide a theoretical guarantee that DECI can recover the ground truth causal graph under standard causal discovery assumptions. Motivated by application impact, we extend this model to heterogeneous, mixed-type data with missing values, allowing for both continuous and discrete treatment decisions. Our results show the competitive performance of DECI when compared to relevant baselines for both causal discovery and (C)ATE estimation in over a thousand experiments on both synthetic datasets and causal machine learning benchmarks across data-types and levels of missingness. 
\end{abstract}

\section{Introduction}

Causal-aware decision making is pivotal in many fields
such as economics \citep{battocchi2021estimating, zhang2006extensions} and healthcare \citep{bica2019estimating, huang2021diagnosis, tu2019causal}. 
For example, in healthcare, caregivers may wish to understand the effectiveness of different treatments given only historical data. They aspire to estimate treatment effects from observational data, with incomplete or no knowledge of the causal relationships between variables.
This is the 
\emph{end-to-end causal inference} problem, displayed in \Cref{fig:overview}, where we discover the causal graph and estimate treatment effects together using weaker causal assumptions and observational data.  %

\begin{figure*}[t]
\centering
  \vspace{-0.5cm}
\resizebox{.95\textwidth}{!}{
\input{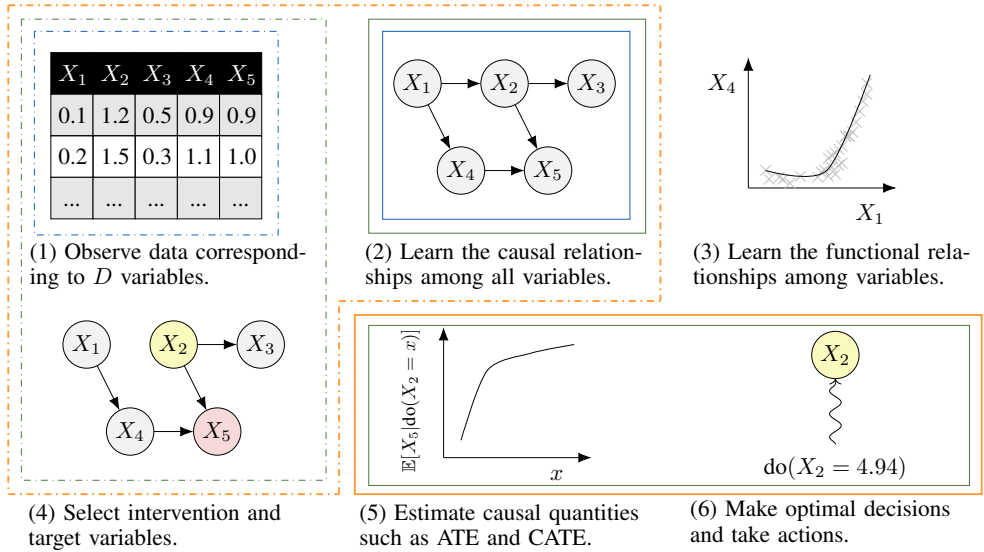}
}
\vspace{-2pt}
    \caption{An overview of the {\color{ferngreen}deep end-to-end causal inference} pipeline compared to traditional {\color{denim}causal discovery} and {\color{deepsaffron}causal inference}. The dashed line boxes show the inputs and the solid line boxes show the outputs. In causal discovery, a user provides observational data (1) as input. The output is the causal relationship (2) which are DAGs or partial DAGs. In causal inference, the user needs to provide both the data (1) and the causal graph (2) as input and provide a causal question by specifying treatment and effect (4), a model is learned and outputs the causal quantities (5) which helps decision making (6). In this work, we aim to answer causal questions end-to-end. DECI allows the user to provide the observational data only and specify any causal questions %
    and output both the discovered causal relationship (2) and the causal quantities (5) that helps decision making (6).   }
    \vspace{-15pt}
    \label{fig:overview}
\end{figure*}

It is well known that any causal conclusion drawn from observational data requires assumptions that are not testable in the observational environment \citep{pearl2009causal}.
Existing methods for estimating causal quantities from data, which we refer to as \emph{causal inference methods}, commonly assume complete \emph{a priori} knowledge of the 
causal graph. This is rarely available in real-world applications, especially when many variables are involved.  On the other hand, existing \emph{causal graph discovery methods}, i.e. those that seek to infer the causal graph from observational data,  require assumptions about statistical properties of the data, which often require less human input
\citep{spirtes1991algorithm}. These methods
often return a large set of plausible graphs, as shown in \Cref{fig:overview}.  
This incompatibility of assumptions and inputs/outputs makes the task of answering causal queries in an \emph{end-to-end manner} non-trivial.

We tackle the problem of end-to-end causal inference (ECI) in a non-linear additive noise structural equation model (SEM) with no latent confounders.
Our framework aims to allow practitioners to estimate causal quantities using only observational data as input.
Our contributions are:
\begin{itemize}[leftmargin=*]
  \setlength\itemsep{2pt}
    \item \textit{A deep learning-based end-to-end causal inference framework named DECI, which performs both causal discovery and inference.} 
    DECI is an autoregressive-flow based non-linear additive noise SEM capable of learning complex nonlinear relationships between variables and non-Gaussian exogenous noise distributions. 
    DECI uses variational inference to learn a posterior distribution over causal graphs. Additionally, we show how the functions learnt by DECI can later be used for simulation-based estimation of (C)ATE. 
    DECI is trained once on observational data; different causal quantities can then be efficiently extracted from the fitted structural equation model.
    
    \item \textit{Theoretical analysis of DECI.}
    We show that, under correct model specification, DECI asymptotically recovers the true causal graph and data generating process.
    Furthermore, we show that DECI generalizes a number of causal discovery methods, such as \notears\ \citep{zheng2018dags, zheng2020learning}, \grandag\ \citep{lachapelle2019gradient}, and others\,\citep{ng2019graph, ng2020role}, providing a unified view of functional causal discovery methods.
    
    \item \textit{Extending DECI for applicability to real data.}
     To make DECI applicable to real-data, we implement support for mixed type (continuous and categorical) variables and missing value imputation.
    
    \item \textit{Insights into ECI performance with more than 1000 experiments.} 
    We systematically evaluate DECI, along with a range of combinations of existing discovery and inference algorithms. DECI performs very competitively with baselines from both the causal discovery and inference domains.
\end{itemize}

\section{Related Work and Preliminaries}

\textbf{Related Work.} Our work relates to both causal discovery and causal inference research.  Approaches for causal discovery from observational data can be classified into three groups: constraint-based, score-based, and functional causal models \citep{glymour2019review}.
Recently, \citet{zheng2018dags} framed the directed acyclic graph (DAG) structure learning problem as a continuous optimisation task. Extensions \citep{lachapelle2019gradient,zheng2020learning} employ nonlinear function approximators, like neural networks, to model the relationships among connected variables. 
Our work combines this class of approaches with standard causal assumptions \citep{ng2019graph} to obtain our main theorem about causal graph learning. We extend functional methods to handle mixed data types and missing values.
Outside of functional causal discovery, functional relationships between variables (see \Cref{fig:overview}(3)) are typically not learned by discovery algorithms \citep{spirtes1991algorithm}. Thus, distinct models, with potentially incompatible assumptions or inputs, must be relied upon for causal inference. However, when a DAG cannot be fully identified given the available data, constraint and score-based methods often return partially directed acyclic graphs (PAGs) or completed partially directed acyclic graphs (CPDAGs) \citep{spirtes2000causation}. Instead of returning a summary graph representing a set, DECI returns a distribution over DAGs in such situation.%

Causal inference methods  
assume that either the graph structure is provided~\citep{pearl2009causal} or relevant structural assumptions are provided without the graph~\citep{imbens2015causal}. 
Causal inference can be decomposed into two steps: identification and estimation. Identification focuses on converting the causal estimand (e.g. $P(Y|\Do{X=x}, W)$) into an estimand that can be estimated using the observed data distribution (e.g. $P(Y|X, W)$). Common examples of identification methods include the back-door and front-door criteria~\citep{pearl2009causal}, and instrumental variables~\citep{angrist1996identification}.
Causal estimation computes the identified estimand using statistical methods, such as simple conditioning, inverse propensity weighting~\citep{li2018balancing}, or matching \citep{rosenbaum1983central, stuart2010matching}. Machine learning-based estimators for CATE have also been proposed~\citep{chernozhukov2018double,wager2018estimation}.
Recent efforts to weaken structural assumption requirements~\citep{guo2021minimal, pmlr-v139-jung21b} allow for PAGs and CPDAGs. Our work takes steps in this direction, allowing inference with distributions over graphs.

\textbf{Structural Equation Models (SEM).} Let $\rvx = (x_1,\dots,x_D)$ be a collection of random variables. SEMs \citep{pearl2009causal} model causal relationships between the individual variables $x_i$.
Given a DAG $G$
on nodes $\{1, \dots, D\}$, $\rvx$ can be described by
    $x_i = F_i\left(\rvx_{\text{pa}(i; G)}, z_i \right)$,
where $z_i$ is an exogenous noise variable that is independent of all other variables in the model, $\text{pa}(i; G)$ is the set of parents of node $i$ in $G$, and $F_i$ specifies how variable $x_i$ depends on its parents and the noise $z_i$. In this paper, we focus on additive noise SEMs, also referred to as additive noise models (ANM), i.e.
\begin{equation}
F_i\left(\rvx_{\text{pa}(i; G)}, z_i \right) = f_i\left(\rvx_{\text{pa}(i; G)}\right) + z_i \quad \text{or} \quad \rvx = f_G(\rvx) + \rvz \quad \text{in vector form}. \label{eq:ANM_vec_f}
\end{equation}

\textbf{Average Treatment Effects.} %
The ATE and CATE quantities allow us to  estimate the impact of our actions (treatments) \citep{pearl2009causal}.
Assume that $\rvx_T$ (with $T \subset \{1,\dots,D\}$) are the treatment variables;
the interventional distribution is denoted $p(\rvx\mid\text{do}(\rvx_T = \rva))$. 
The ATE and CATE on targets $\rvx_Y$ for treatment $\rvx_T{=}\rva$ given a reference $\rvx_T{=}\rvb$, and conditional on $\rvx_C{=}\rvc$ for CATE, are given by
\begin{gather}\label{eq:ATE}
    \text{ATE}(\rva,\rvb) = \E_{p(\rvx_Y\mid \text{do}(\rvx_T=\rva))}[\rvx_Y] - \E_{p(\rvx_Y\mid \text{do}(\rvx_T=\rvb))}[\rvx_Y],\quad \text{and}\\
    \text{CATE}(\rva,\rvb |\rvc) = \E_{p(\rvx_Y\mid \text{do}(\rvx_T=\rva), \rvx_C=\rvc)}[\rvx_Y] - \E_{p(\rvx_Y\mid \text{do}(\rvx_T=\rvb), \rvx_C=\rvc)}[\rvx_Y].
\end{gather}
We consider the common scenario where the \emph{conditioning variables are not caused by the treatment}.%

\section{DECI: Deep End-to-end Causal Inference}

We introduce DECI, an end-to-end deep learning-based causal inference framework.
DECI learns a distribution over causal graphs from observational data and (subsequently) estimates causal quantities.
\Cref{sec:deci_cd}, describes our autoregressive flow based ANM SEM.
\Cref{sec:deci_consistency} lays out the conditions under which DECI will recover the true causal graph given enough observational data (\Cref{thm: consistency of DECI}). \Cref{sec:deci_ate} shows how the generative model learnt by DECI can be used to simulate samples from intervened distributions, allowing for treatment effect estimation.
\Cref{sec:deciforthewin} extend's DECI's real-world applicability by adding support for non-Gaussian exogenous noise, mixed type data (continuous and discrete), and imputation for partially observed data.

\subsection{DECI and Causal Discovery} \label{sec:deci_cd}

DECI takes a Bayesian approach to causal discovery \citep{heckerman1999bayesian}. We model the causal graph $G$ jointly with the observations $\rvx^1, \hdots, \rvx^N$ as
\begin{equation}
p_\theta(\rvx^1, \hdots, \rvx^N, G) = p(G) \prod_n p_\theta(\rvx^n|G). \label{eq:joint_model}
\end{equation}
We aim to fit $\theta$, the parameters of our non-linear ANM, using observational data. Once this model is fit, the posterior $p_\theta(G|\rvx^1,\hdots,\rvx^N)$ characterizes our beliefs about the causal structure.

\textbf{Prior over Graphs. }
The graph prior $p(G)$ should characterize the graph as a DAG. We implement this by leveraging the continuous DAG penalty from \citet{zheng2018dags},
\begin{equation}
h(G) = \mathrm{tr}\left(e^{G\odot G}\right) - D, \label{eq:dagness}
\end{equation}
which is non-negative and zero only if $G$ is a DAG. We then implement the prior as
\begin{equation}
    p(G) \propto \exp \left(-\lambda_s \Vert G\Vert_F^{2} - \rho \, h(G)^2 - \alpha \, h(G)\right), 
    \label{eq: soft prior}
\end{equation}
where we weight the DAG penalty by $\alpha$ and $\rho$. These are gradually increased during training following an augmented Lagrangian scheme, ensuring only DAGs remain at convergence. We introduce prior knowledge about graph sparseness by penalising the norm $\lambda_{s}\|G\|_{F}$, with $\lambda_{s}$ a scalar.

\textbf{Likelihood of Structural Equation Model.}
Following \citet{khemakhem2021causal}, we factorise the observational likelihood $p_\theta(\rvx^n|G)$ in an autoregressive manner.
Rearranging the ANM assumption \cref{eq:ANM_vec_f}, we have $\rvz=g_G(\rvx; \theta)=\rvx-f_G(\rvx; \theta)$.
The components of $\rvz$ are independent. If we have a distribution $p_{z_i}$ for component $z_i$, then we can write the observational likelihood as
\begin{equation}
    p_\theta(\rvx^n|G) = p_\rvz\left(g_G(\rvx^n; \theta)\right) =  \prod_{i=1}^D p_{z_i}\left(g_G(\rvx^n; \theta)_i\right), \label{eq:likelihoodterm}
\end{equation}
where we omitted the Jacobian-determinant term
because it is always equal to one for DAGs $G$ \citep{mooij2011causal}.
The choice of $f_G:\mathbb{R}^d \rightarrow \mathbb{R}^d$ must satisfy the adjacency relations specified by the graph $G$. If there is no edge $j \rightarrow i$ in $G$, then the function $f_i(\rvx)$---the $i$-th component of the output of $f_G(\rvx)$--- must satisfy $\partial f_i(\rvx)/\partial x_j = 0$. We propose a flexible parameterization that satisfies this by setting
\begin{equation}\label{eq:fcause_functions}
    f_i(\rvx) = \zeta_i\left(\sum_{j=1}^d G_{j, i} \,\, \ell_j(x_j) \right),
\end{equation}
where $G_{j,i}\in\{0, 1\}$ indicates the presence of the edge $j\rightarrow i$, and $\ell_i$ and $\zeta_i$ ($i=1,\hdots,d$) are MLPs. A na{\"i}ve implementation would require training $2D$ neural networks. Instead, we construct these MLPs so that their weights are shared across nodes as $\zeta_i(\cdot) = \zeta(\mathbf{u}_i, \cdot)$ and $\ell_i(\cdot) = \ell(\mathbf{u}_i, \cdot)$, with $\mathbf{u}_i\in\mathbb{R}^D$ a trainable embedding that identifies the source and target nodes respectively. %

\textbf{Exogenous Noise Model $p_{\rvz}$.} 
We consider two possible models for the distribution of $\rvz$. 1) A simple Gaussian $p_{z_i}(\cdot) = \mathcal{N}\left(\cdot \vert 0, \sigma_i^2 \right)$, where per-variable variances $\sigma_i^2$ are learnt.
2) A flow \cite{rezende2015variational}
\begin{gather}\label{eq:noise_model}
  p_{z_i}(z_{i}) = \mathcal{N}\left(\kappa^{-1}_i(z_{i}) \vert 0, 1\right) \left|\frac{\partial \kappa^{-1}_i(z_{i})}{\partial z_{i} }\right|.
\end{gather}
We choose the learnable bijections $\kappa_i$ to be a rational quadratic splines \cite{Durkan2019spline}, parametrised independently across dimensions. We do not couple across dimensions since our SEM requires independent noise variables. 
Spline flows are significantly more flexible than the Gaussian distributions employed in previous work \cite{lachapelle2019gradient, ng2019graph, ng2020role, zheng2018dags, zheng2020learning}.

\textbf{Optimization and Inference Details. }
The model described presents two challenges. First, the true posterior over $G$ is intractable. Second, maximum likelihood cannot be used to fit the model parameters, due to the presence of the latent variable $G$. We simultaneously overcome both of these challenges using variational inference \citep{blei2017variational, jordan1999introduction,  zhang2018advances}. We define a variational distribution $q_\phi(G)$ to approximate the intractable posterior $p_\theta(G|\rvx^1,\hdots, \rvx^N)$, and use it to build  the ELBO, given by
\begin{equation}
    \mathrm{ELBO}(\theta, \phi) = \mathbb{E}_{q_\phi(G)} \left[ \log p(G) \prod_n p_\theta(\rvx^n|G) \right] + H(q_\phi)
    \leq \log p_\theta(\rvx^1, \hdots, \rvx^N), 
    \label{eq:DECIELBO}
\end{equation}
where $H(q_\phi)$ represents the entropy of the distribution $q_\phi$ and $p_\theta(\rvx^n|G)$ takes the form of \cref{eq:likelihoodterm} (derivation in \Cref{sec:elboderivation}). We choose $q_\phi(G)$ to be the product of independent Bernoulli distributions for each potential directed edge in $G$.
We parametrize edge existence and edge orientation separately, using the ENCO parametrization \cite{lippe2021efficient}.
The SEM parameters $\theta$ and variational parameters $\phi$ are trained by maximizing the ELBO. The Gumbel-softmax trick \citep{jang2016categorical, maddison2016concrete} is used to stochastically estimate the gradients with respect to $\phi$. \Cref{app:algorithm} details the full optimisation procedure.

\textbf{Unified View of functional Causal Discovery. } We note that, like DECI, many functional causal discovery methods \citep{lachapelle2019gradient, ng2019graph, ng2020role, zheng2018dags, zheng2020learning} can be seen from a probabilistic perspective as fitting an autoregressive flow (with a hard acyclicity constraint) for different choices for the exogenous noise distribution $p_\rvz$ and transformation function $f$. We expand on the details of this perspective and formalising it in \Cref{app:unified_flow}. DECI employs NNs for $f$ and flexible, potentially non-Gaussian, distributions for $p_\rvz$, making it the most flexible member of this family.

\subsection{Theoretical Considerations for DECI} \label{sec:deci_consistency}

We now show that maximizing the ELBO from \cref{eq:DECIELBO} recovers both the ground truth data generating process $p(\bm{x};G^0)$ and true causal graph $G^0$ in the infinite data limit. This is formalized in \Cref{thm: consistency of DECI}. The assumptions required by the theorem, which are common in causal discovery research, can be informally summarized as (formal assumptions in \Cref{app: FCause Consistency}):
\vspace{-0.3cm}
\begin{itemize}[leftmargin=*]
  \setlength\itemsep{0pt}
    \item \emph{Minimality} and \emph{Structural Identifiability}, satisfied by a continuous non-linear ANM 
    \citep{peters2010identifying},
    
    \item \emph{Correct Specification}, there exists $\theta^*$ such that $p_{\theta^*}(\rvx;G^0)$ matches the data-generating process,
    
    \item \emph{Causal Sufficiency}, there are no latent confounders,
    
    \item \emph{Regularity of log likelihood}, for all $\theta$ and $G$ we have  $\mathbb{E}_{p(\rvx;G^0)}\left[\vert\log p_\theta(\rvx;G)\vert\right]<\infty$.
\end{itemize}
\vspace{-0.12cm}
\begin{theorem}[DECI recovers true data generating process]
Under assumptions 1-5 (\Cref{app: FCause Consistency}), the solution $(\theta',q'_\phi({G}))$ from maximizing the ELBO (\cref{eq:DECIELBO}) satisfies $q'_\phi({G})=\delta({G}={G}')$ where ${G}'$ is a unique graph. In particular, ${G}'={G}^0$ and $p_{\theta'}(\mathbf{x};{G}')=p(\mathbf{x};{G}^0)$. 
\label{thm: consistency of DECI}
\end{theorem}
\vspace{-0.12cm}
The proof is in \Cref{app: FCause Consistency}. It consists of two key steps: (i) the maximum likelihood estimate (MLE) of $(\theta,G)$ recovers $p(\bm{x};G^0)$, (ii) solutions from maximizing the ELBO approach the MLE in the large data limit. 
Specifically, we show that DECI induces the same joint likelihood as the ground truth and the posterior $q_\phi({G})$ is a delta function $\delta(G=G^0)$ concentrated on the true graph $G^0$.

\subsection{Estimating Causal Quantities} \label{sec:deci_ate}

We now show how the generative model learnt by DECI can be used to evaluate expectations under interventional distributions, and thus estimate ATE and CATE.
As explained above, DECI returns $q_\phi(G)$, an approximation of the posterior over graphs given observational data.  
Then, interventional distributions and treatment effects can be obtained by marginalizing over graphs as
\begin{gather}
 \E_{q_\phi(G)}\left[p\left(\rvx_{Y} | \Do{\rvx_{T}{=}\rva},  G\right)\right],\quad
    \E_{q_\phi(G)}[ \text{ATE}(\rva, \rvb | G)]\quad \mbox{and} \quad
    \E_{q_\phi(G)}[ \text{CATE}(\rva, \rvb | \rvc, G)]. \label{eq:cate_ref}
    \vspace{-11pt}
\end{gather}
This can be seen as a probabilistic relaxation of traditional causal quantity estimators. When have observed enough data to be certain about the causal graph, i.e.  $q_\phi(G) = \delta(G-G_{i})$, our procedure matches traditional causal inference. 
We go on to discuss how DECI estimates (C)ATE.

\begin{wrapfigure}{r}{0.43\textwidth}
\vspace{-1.1cm}
  \begin{center}
    \includegraphics[width=\linewidth]{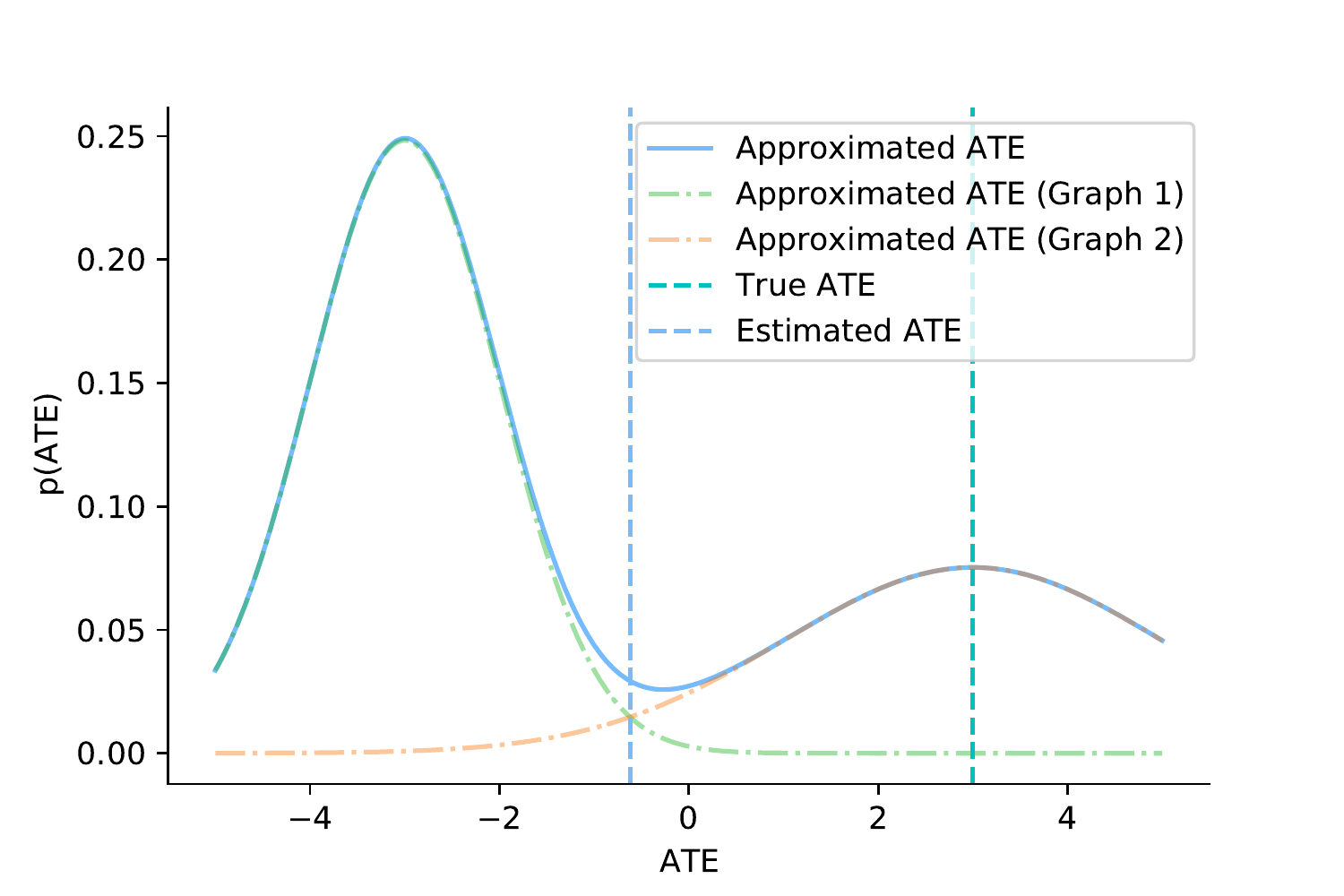}
  \end{center}
  \vspace{-.25cm}
  \caption{Marginalizing ATE estimates from graph and exogenous noise samples.}
    \label{fig:multimodal_ate}
\end{wrapfigure}

\textbf{Estimating ATE. }
After training, we can use the model learnt by DECI to simulate new samples $\rvx$ from $p_{\theta}(\rvx | G)$. We sample a graph $G \sim q_{\phi}(G)$ and a set of exogenous noise variables $\rvz \sim p_{\rvz}$. We then input this noise into the learnt DECI structural equation model to simulate $\rvx$, by applying \cref{eq:ANM_vec_f} and \cref{eq:fcause_functions} on $\rvz$ in the topological order defined by $G$. 
However, ATE estimation requires samples from the interventional distribution $p\left(\rvx_{\setminus T} | \Do{\rvx_{T}=\rvb}, G\right)$. These can be obtained by noting that
\begin{gather}
    p\left(\rvx_{\setminus T} | \Do{\rvx_{T}=\rvb}, G\right) = p\left(\rvx_{\setminus T} | \rvx_{T}=\rvb, G_{\Do{\rvx_{T}}}\right), \nonumber
\end{gather}
where $G_{\Do{\rvx_{T}}}$ is the ``mutilated'' graph obtained by removing incoming edges to $\rvx_{T}$. Thus, samples from this distribution can be obtained by following the sampling procedure explained above, but fixing the values $\rvx_{T}=\rvb$ and using $G_{\Do{\rvx_{T}}}$ instead of $G$. Finally, we use these samples to obtain a Monte Carlo estimate of the expectations required for ATE computation \cref{eq:ATE}. 
\Cref{fig:multimodal_ate} illustrates that these samples are from a mixture distribution when the posterior has not collapsed to one graph. 

\textbf{Estimating CATE. }
We focus on CATE estimation for which the treatment $\rvx_T$ is not the cause of the conditioning set $\rvx_C$, i.e. there is no directed path from $T$ to $C$ in $G$.
Under this assumption, we can estimate CATE by sampling from the interventional distribution $p\left(\rvx_{\setminus T} | \Do{\rvx_{T}}, G\right)$ and then estimating the conditional distribution of $\rvx_Y$ given $\rvx_C$.
To make this precise, we let $Y = X\setminus{(T \cup C)}$ denote all variables that we do not intervene or condition on. Conditional densities
\begin{equation}\label{eq:conditional_densities}
    p_{\theta}(\rvx_Y\mid \text{do}(\rvx_T{=}\rvb), \rvx_C{=}\rvc, G) = \frac{p_{\theta}(\rvx_Y, \rvx_C{=}\rvc | \rvx_T{=}\rvb, G_{\text{do}(\rvx_T)})}{ p_{\theta}(\rvx_C{=}\rvc | \rvx_T{=}\rvb,  G_{\text{do}(\rvx_T)})}
\end{equation}
are not directly tractable in the DECI model
due to the intractability of the marginal $p_{\theta}(\rvx_C{=}\rvc | \rvx_T{=}\rvb,  G_{\text{do}(\rvx_T)})$. However, we can always sample from the joint interventional distribution $p_{\theta}(\rvx_Y, \rvx_C | \rvx_T{=}\rvb,  G_{\text{do}(\rvx_T)})$.
We use samples from this joint to train a surrogate regression model $g_{G}$\footnote{Subscript $G$ allows differentiating surrogate models fit on samples from different graphs drawn from $q_\phi$.} to the relationship between $\rvx_{C}$ and $\rvx_{Y}$. Specifically, we minimize the square loss
\begin{gather}
   \E_{ p_{\theta}(\rvx_Y, \rvx_C | \rvx_T{=}\rvb,  G_{\text{do}(\rvx_T)})} \left[\Vert \rvx_{Y} -  g_{G}(\rvx_{C}) \Vert^2 \right]. \nonumber
\end{gather}
making $g_G$ approximate the conditional mean of $\rvx_Y$.
We choose $g_{G}$ to be a basis-function linear model with random Fourier basis functions \citep{Felix2016orthogonal}. As illustrated in \Cref{fig:cate_illustration}, we train two separate surrogate models, one for our intervention $\rvx_{T} = \rva$ and one for the reference $\rvx_{T} = \rvb$.
We estimate CATE as the difference between their outputs evaluated at $\rvx_{C}=\rvc$. This process is repeated for multiple posterior graphs samples $G\sim q_{\phi}$, allowing us to marginalise the posterior graphs
\begin{equation}
    \E_{q_\phi(G)}\left[ g_{G_{\text{do}(\rvx_T{=}\rva)}}(\rvx_{C}=\rvc) - g_{G_{\text{do}(\rvx_T{=}\rvb)}}(\rvx_{C}=\rvc))\right].
\end{equation}

\textbf{General ECI Framework. }
The probabilistic treatment of the DAG, and the re-use of functional causal discovery generative models for simulation-based causal inference are principles that can be applied beyond DECI. 
Constraint-based \cite{spirtes1991algorithm} and score-based \citep{chickering2002optimal} discovery methods often output a set of DAGs compatible with the data, i.e. a PAG or CPDAG. It is natural to interpret these equivalence classes as uniform distributions over members of sets of graphs.
We can then use \cref{eq:cate_ref} to estimate causal quantities by marginalizing over these distributions.
The quantities inside the expectations over graphs can be estimated using any existing causal inference method, such as linear regression \citep{sharma2021dowhy},  Double ML \citep{chernozhukov2018double}, etc. 
Our experiments explore combinations of discovery methods that return graph equivalence classes with standard causal inference methods.
We take expectations over causal graphs since these return the quantity that minimises the posterior expected squared error in our (C)ATE estimates while noting that the best statistic will be application dependent.

\begin{figure*}[t]
    \vspace{-0.5cm}
    \centering
    \resizebox{0.95\textwidth}{!}{
    \input{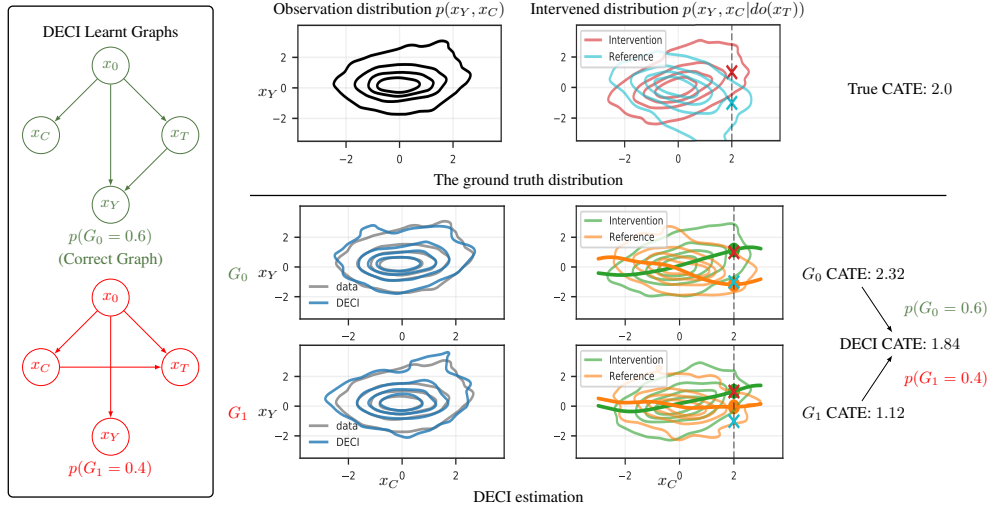}
    }
    \vspace{-0.1cm}
    \caption{DECI CATE estimation on the CSuite Symprod Simpson dataset. Left: The DECI graph posterior has two modes with $p(G)=0.6$ for the correct graph and $p(G)=0.4$ for an alternative possibility with some incorrect edges.
    Middle: we display the joint distribution of conditioning and effect variables in the observational setting and under interventions on $\rvx_{T}$. DECI captures the observational density well.
    Right: interventional distributions with their conditional means $\rvx_{C}=\rvc$ marked with crosses. DECI predicts conditional expectations by fitting functions from $\rvx_{C}$ to $\rvx_{Y}$ and evaluating them at $\rvc$. DECI outputs CATE by marginalizing the result over possible graphs. 
    }
    \vspace{-10pt}
    \label{fig:cate_illustration}
\end{figure*}
\subsection{DECI for Real-world Heterogeneous Data} \label{sec:deciforthewin}
We extend DECI to handle mixed-type (continuous and discrete) data and data with missing values, which often arise in real-world applications.

\textbf{Handling Mixed-type Data.}
For discrete-valued variables, we remove the additive noise structure and directly parameterise parent-conditional class probabilities 
\begin{equation}
    p^\text{discrete}_\theta\left(x_i | \rvx_{\Pa{i;G}};G\right) = P_i\left(\rvx_{\Pa{i;G}};\theta\right)(x_i),
\end{equation}
where $P_i\left(\rvx_{\Pa{i;G}};\theta\right)$ is a normalised probability mass vector over the number of classes of  $x_i$, obtained by applying the softmax operator to $f_i(\rvx_{\Pa{i;G}})$. This means that for discrete variables, the output of $f_i$ is a vector of length equal to the number of classes for variable $i$.
This approach gives a valid likelihood for $p_\theta(\rvx^n|G)$ which we use to train DECI. However, since the full generative model is no longer an ANM, we cannot guarantee that \Cref{thm: consistency of DECI} applies in this setting.

\textbf{Handling Missing Data.} We propose an extension of DECI to partially observed data.\footnote{We assume that values are missing (completely) at random, the most common setting \citep{ma2018eddi,rubin1976inference, stekhoven2012missforest, strobl2018fast}.} We use $\rvx^n_o$ to denote the observed components of $\rvx^n$, $\rvx^n_u$ to denote the unobserved components, and their joint density in the observational environment is $p_\theta(\rvx^n_o, \rvx^n_u)$. We approximate the posterior $p(G, \rvx^n_u|\rvx^n_o)$ with the variational distribution,
\begin{equation}
q_{\phi, \psi}\left(G, \rvx^1_u, \hdots, \rvx^N_u|\rvx^1_o, \hdots, \rvx^N_o\right) = q_\phi(G) \prod_n q_{\psi}(\rvx^n_u|\rvx^n_o),  \nonumber
\end{equation}
which yields the following learning objective
\begin{equation}
\vspace{-1pt}
\mathrm{ELBO}(\theta, \phi, \psi) = H(q_\phi) + \sum_n H(q_{\psi}(\rvx_{u}^{n} | \rvx_{o}^{n})) + \mathbb{E}_{q_{\phi, \psi}} \left[\log p(G) \prod_n p_\theta(\rvx^n_o,  \rvx^n_u|G)\right].\label{eq:notfullyobs}
\end{equation}
We parameterize the Gaussian imputation distribution $q_{\psi_n}(\rvx^n_u|\rvx^n_o)$ using an amortization network \citep{kingma2013auto}, whose input is $\rvx^n_o$, and output the mean and variance of the imputation distribution $q_{\psi}(\rvx^n_u|\rvx^n_o)$.

\section{Experiments}

We evaluate DECI on both causal discovery and causal inference tasks.
A full list of results and details of the experimental set-up are in Appendices~\ref{sec:DECIdets} and \ref{sec:app:additional}.
Our code is in the supplement.

\subsection{Causal Discovery Evaluation} \label{sec:CD_exp}

\begin{wrapfigure}{r}{0.57\textwidth}
\vspace{-1.5cm}
  \begin{center}
    \includegraphics[width=\linewidth, trim = {0 3.5cm 0.3cm 0}, clip]{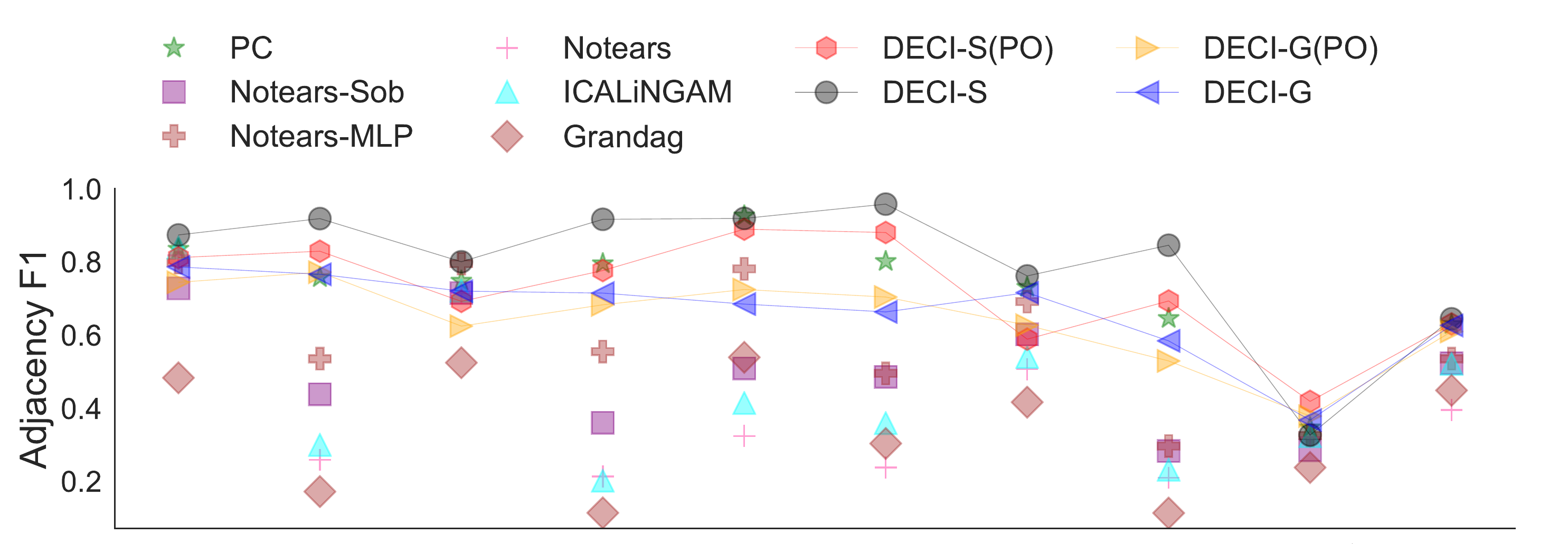}
    
    \includegraphics[width=\linewidth, trim = {0 3.5cm 0.3cm 2.95cm}, clip]{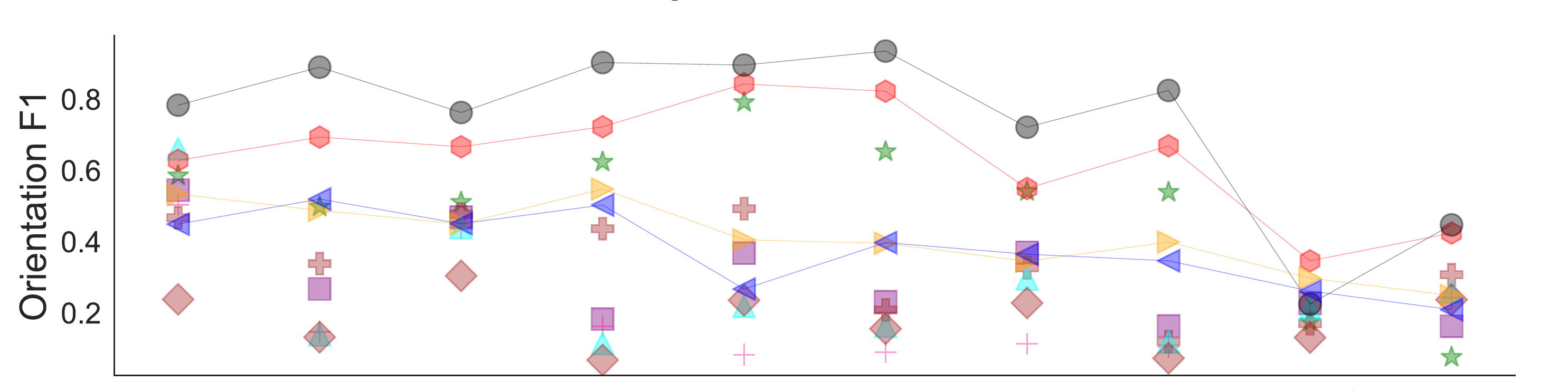}
    
    \includegraphics[width=\linewidth, trim = {0 0 0.3cm 2.95cm}, clip]{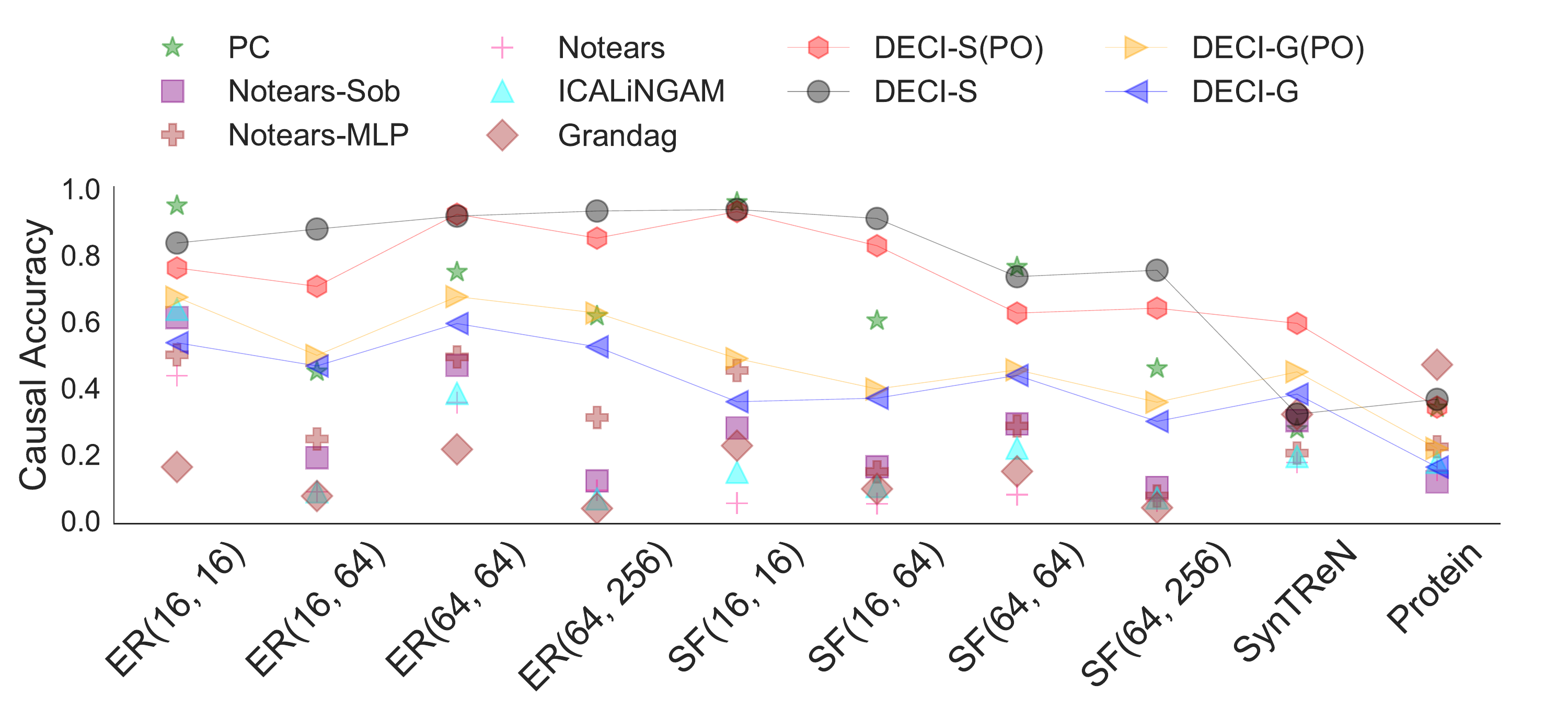}
  \end{center}
  \vspace{-0.4cm}
  \caption{Causal discovery on benchmark datasets. 
    The label \textit{(PO)} corresponds to running DECI with 30\% of the training data missing. For readability, the DECI results by are connected with with soft lines. The figure shows mean results across five different random seeds.}
    \label{fig:cd_summary}
    \vspace{-0.5cm}
\end{wrapfigure}

\textbf{Datasets.}
We consider synthetic, pseudo-real, and real data. For the synthetic data, we follow \citet{lachapelle2019gradient} and \citet{zheng2020learning} by sampling a DAG from two different random graph models, \textbf{Erd\H{o}s-R\'{e}nyi (ER)} and \textbf{scale-free (SF)}, and simulating each ANM $x_i = f_i(\rvx_{\Pa{i;G}}) + z_i$, where $f_i$ is a nonlinear function (randomly sampled spline). %
We consider two noise distributions for $z_i$, a standard Gaussian and a more complex one obtained by transforming samples from a standard Gaussian with an MLP with random weights. We consider number of nodes $d \in \{16, 64\}$ with number of edges
$e \in \{d, 4d\}$. The resulting datasets are identified as \textbf{ER$(d, e)$} and \textbf{SF$(d, e)$}. All datasets have $n{=}5000$ training samples.

For the pseudo-real data we consider the \textbf{SynTReN} generator \citep{van2006syntren}, which creates synthetic transcriptional regulatory networks and produces simulated gene expression data that mimics experimental data. We use the datasets generated by \cite{lachapelle2019gradient} ($d{=}20$), and take $n{=}400$ for training. Finally, for the real dataset, we use the protein measurements in human cells from \citet{sachs2005causal}. 
We use a training set with $n{=}800$ observational samples and $d{=}11$.

\textbf{Baselines.} We run DECI using two models for exogenous noise: a Gaussian with learnable variance (identified as DECI-G) and a spline flow (DECI-S). We compare against \textit{PC} \citep{kalisch2007estimating}, (linear) \notears\ \citep{zheng2018dags}, the nonlinear variants \notearsmlp\ and \notearssob\ \citep{zheng2020learning}, \grandag\ \citep{lachapelle2019gradient}, and \textit{ICALiNGAM} \citep{shimizu2006linear}. When a CPDAG is the output, e.g., from PC, we treat all possible DAGs under the CPDAG as having the same probability. All baselines are implemented with the 
\texttt{gcastle} package \citep{zhang2021gcastle}.

\textbf{Causality Metrics.} We report F1 scores for adjacency, orientation \citep{glymour2019review, tu2019causal} and causal accuracy \citep{claassen2012bayesian}. For DECI, we report the expected values of these metrics estimated over the graph posterior.

\Cref{fig:cd_summary} shows the results for the data generated with non-Gaussian noise. We observe that DECI achieves the best results across all metrics. Additionally, using the flexible spline model for the exogenous noise (DECI-S) yields better results than the Gaussian model (DECI-G). This is expected, as the noise used to generate the data is non-Gaussian. For Gaussian noise (see \Cref{fig:cd_summary_gauss}), both DECI-S and DECI-G perform similarly. Moreover, when data are partially observed \textit{(PO)}, the strong performance of DECI remains, showing that DECI can handle missing data efficiently.

\subsection{End-to-end Causal Inference}

We evaluate the \emph{end-to-end} pipeline, taking in observational data and returning (C)ATE estimates.

\textbf{Datasets.}
We generate ground-truth treatment effects to compare against for the \textbf{ER} and \textbf{SF} synthetic graphs that were described in \Cref{sec:CD_exp} by applying random interventions on these synthetic SEMs, ensuring at most 3 edges between the intervention and effect variables.
For more detailed analysis, we hand-craft a suite of synthetic SEMs, which we name \textbf{CSuite}. CSuite datasets elucidate particular features of the model, such as identifiability of the causal graph, correct specification of the SEM, exogenous noise distributions, and size of the optimal adjustment set. 
We draw conditional samples from CSuite SEMs with HMC, allowing us to evaluate CATE. 
Finally, we include two semi-synthetic causal inference benchmark datasets for ATE evaluation: \textbf{Twins} (twin birth datasets in the US) \citep{almond2005costs} and \textbf{IHDP} (Infant
Health and Development Program data) \citep{hill2011bayesian}. 
See \Cref{app:benchmark} for all experimental details.

\setlength{\tabcolsep}{1.5pt}
\begin{wraptable}{r}{0.57\textwidth}
    \vspace{-0.44cm} 
    \centering
    \caption{Method rank on different (CSuite, Twins, IHDP and ER/SF) datasets, ranking by median ATE RMSE. We present mean $\pm 1$ s.e. of the rank over 27 datasets. Supporting data in \Cref{tab:ate_rmse}. Bold indicates the possible top methods, accounting for error bars. We treat methods with access to the true graph separately. 
    }
    \small
    \begin{tabular}{lr} \toprule
        Method & Mean rank \\
        \midrule
        DECI Gaussian (DGa)  & $\mathbf{6.26 \pm 0.60}$ \\
        DECI Gaussian DoWhy Linear (DGa+L)  & $8.37 \pm 0.50$ \\
        DECI Gaussian DoWhy Nonlinear (DGa+N)  & $8.52 \pm 0.51$ \\
        DECI Spline (DSp)  & $\mathbf{6.04 \pm 0.68}$ \\
        DECI Spline DoWhy Linear (DSp+L)  & $7.78 \pm 0.60$ \\
        DECI Spline DoWhy Nonlinear (DSp+N)  & $\mathbf{6.63 \pm 0.66}$ \\
        PC + DoWhy Linear (PC+L)  & $8.87 \pm 0.41$ \\
        PC + DoWhy Nonlinear (PC+N)  & $7.54 \pm 0.45$ \\
        \midrule
        True graph DECI Gaussian (T+DGa)  & $\mathbf{3.74 \pm 0.47}$ \\
        True graph DECI Spline (T+DSp)  & $\mathbf{4.19 \pm 0.56}$ \\
        True graph DoWhy Linear (T+L)  & $4.87 \pm 0.58$ \\
        True graph DoWhy Nonlinear (T+N)  & $5.20 \pm 0.71$\\ \bottomrule
    \end{tabular}
    \label{tab:ate_rmse_rank}
\end{wraptable}

\textbf{Baselines.}
To thoroughly evaluate end-to-end inference, we consider different ways of \emph{combining} discovery and inference algorithms. For DECI, we can use a trained model to immediately estimate (C)ATE. We also consider using the learned DECI graph posterior in combination with existing methods for causal inference on a known graph: {DoWhy-Linear} and {DoWhy-Nonlinear} \citep{sharma2021dowhy} which implement linear adjustment and Double Machine Learning (DML) \citep{chernozhukov2018double} methods for backdoor adjustment respectively.
We also pair other \emph{discovery} methods with DECI and DoWhy treatment effect estimation, namely the PC algorithm as a baseline and the ground truth graph (when available) as a check.
We evaluate end-to-end causal inference on all valid combinations that arise from combining \emph{discovery} methods in \{DECI-Gaussian (DGa), DECI-Spline (DSp), PC, and True graph (T)\} with causal \emph{inference} methods in \{DECI-Gaussian (DGa), DECI-Spline (DSp), DoWhy-Linear (L), DoWhy-Nonlinear (N)\}. 

\textbf{Metrics.} We report RMSE between (C)ATE estimates and the ground truth.

\Cref{tab:ate_rmse_rank} provides a high-level summary of our results. For each dataset, we estimated the ATE using each combination of methods, computed the RMSE and took the median over random seeds. We then ranked methods for each dataset (with 1 being the best) and aggregated over the 27 datasets. We find that DECI Spline has the overall best (lowest) rank. RMSE scores are in \Cref{sec:app:additional}.

\begin{figure}
\vspace{-1.0cm}
    \centering
    \includegraphics[width=0.98\linewidth]{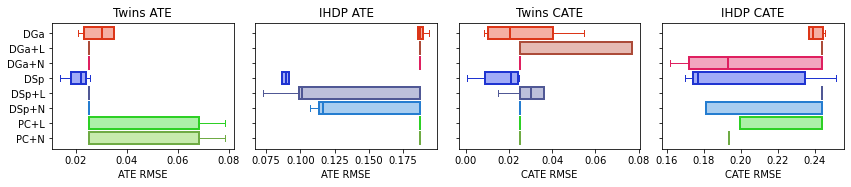}
    \vspace{-5pt}
    \caption{Box plots showing end-to-end ATE and CATE estimation error on the semi-synthetic Twins and IHDP datasets with different method combinations. Method acronyms are as in Table~\ref{tab:ate_rmse_rank}.
    }%
    \label{fig:twins_ihdp}
    \vspace{-0.5cm}
\end{figure}

In \Cref{tab:csuite_anm}, we present detailed results for six CSuite datasets. \textbf{Lin. Exp} is a two node linear SEM with exponential noise, only DECI Spline can recover the true graph, ATE estimation quality is similar for different estimator once the true graph is found. \textbf{Nonlin. Gauss} is a two node non-linear SEM with Gaussian noise, only DECI can fit the highly non-linear functional relationship, with equal performance between DECI-Gaussian and -Spline. \textbf{Large backdoor} is a larger nonlinear SEM with non-Gaussian noise in which adjusting for all confounders is valid, but of high variance. For DECI-Spline, which performs well on discovery, the ATE estimation is best using DECI, as DoWhy takes the maximal adjustment set thereby increasing estimator variance.
\textbf{Weak arrows} is a similar SEM to Large backdoor, except that a maximal adjustment set is now necessary. Here, DECI-Spline is best for discovery, but is somewhat less accurate for ATE estimation given the right graph. \textbf{Nonlin. Simpson} is an adversarially constructed dataset where 1) the true graph is theoretically identifiable (it is a non-linear ANM), but difficult to discover in practice, 2) ATE estimation is very poor given the wrong graph (Simpson's paradox). All methods perform equally badly. \textbf{Symprod Simpson} is a similar but slightly easier dataset, for which DECI-Spline with DML does well.

\begin{wraptable}[18]{r}{0.57\textwidth}
    \vspace{-0.45cm}
    \centering
    \caption{Median ATE RMSEs from 20 seeds for six CSuite datasets.
    }
    \small
    \begin{tabular}{lcccccc}
\toprule
&     Lin. Exp & \shortstack{Nonlin. \\ Gauss} & \shortstack{Large\\ backdoor} & \shortstack{Weak \\ arrows} & \shortstack{Nonlin.\\ Simpson} & \shortstack{Symprod\\ Simpson} \\
\midrule
DGa                 &         1.029 &              \textbf{0.042} &                 0.213 &              1.097 &                 1.995 &                  0.318 \\
DGa+L    &         1.031 &              1.522 &                 0.144 &              1.108 &                 1.994 &                  0.695 \\
DGa+N &         1.031 &              1.532 &                 0.331 &              1.108 &                 1.994 &                  0.487 \\
DSp                   &         0.022 &              \textbf{0.043} &                 \textbf{0.031} &              0.189 &                 1.997 &                  0.427 \\
DSp+L     &         \textbf{0.001} &              1.522 &                 0.091 &              0.110 &                 1.994 &                  0.819 \\
DSp+N   &         0.002 &              1.532 &                 0.232 &              \textbf{0.064} &                 1.994 &                  \textbf{0.160} \\
PC+L              &         0.516 &              1.532 &                 1.690 &              1.108 &                 1.994 &                  0.487 \\
PC+N            &         0.517 &              1.532 &                 1.690 &              1.108 &                 1.994 &                  0.487 \\
\midrule
T+DGa      &         0.073 &              0.034 &                 0.167 &              0.255 &                 0.404 &                  0.101 \\
T+DSp        &         0.028 &              0.034 &                 0.035 &              0.128 &                 0.531 &                  0.242 \\
T+L       &         0.001 &              1.522 &                 0.105 &              0.109 &                 0.848 &                  0.819 \\
T+N    &         0.003 &              1.532 &                 0.241 &              0.015 &                 0.597 &                  0.168 \\
\bottomrule
\end{tabular}
\label{tab:csuite_anm}
\end{wraptable}

We performed similar analysis for ATE estimation on ER and SF datasets, on additional CSuite datasets that contain discrete variables or are not theoretically identifiable, and CATE estimation on a subset of CSuite. See \Cref{sec:app:additional}.

On the semi-synthetic benchmark datasets, Twins and IHDP, we evaluated both ATE and CATE estimation as shown in \Cref{fig:twins_ihdp}. For ATE estimation, DECI-Spline is fractionally better than baselines on Twins and significantly better for IHDP. On IHDP, it appears that only DECI-Spline was successful at causal discovery, and given the right graph, DECI-Spline is the best method for computing ATE. For CATE estimation, a similar pattern. %

\textbf{Summary}
Across all experiments we see that DECI enables end-to-end causal inference with competitive performance on both synthetic and more realistic data. DECI particularly performs well compared to other methods when its ability to handle nonlinear functional relationship and non-Gaussian noise distributions  comes into play in causal discovery \emph{or} causal inference. Other ECI method combinations can achieve strong performance, but have weak performance if either step's assumptions are violated. We find DECI-Spline particularly attractive given its high degree of flexibility---it generally performs on par with or better than other methods. %

\vspace{-3pt}
\section{Discussion, scope, and limitations}\label{sec:discussion}

\emph{Causal inference} requires causal assumptions on the relationships between variables of interest. 
The field of \emph{causal discovery} aims to learn about these relationships from observational data, given some non-causal assumptions on the data generating process.
Motivated by a real-world application where our knowledge of causal relationships is incomplete, DECI combines ideas from causal discovery and inference to go directly from observations to causal predictions.
This formulation requires us to adopt assumptions, %
namely, that the data is generated with a non-linear ANM and that there are no unobserved confounders. Empirically, we find DECI to perform well when these assumptions are satisfied, validating the viability of an end-to-end approach. However, the non-linear ANM assumptions made by DECI are impossible to check in most real-world scenarios. Thus, combining the output of discovery methods with incomplete causal assumptions is an attractive avenue to make end-to-end methods more robust in the future. Interestingly, even in our experiments where DECI's assumptions are violated (missing data, discrete type observations, etc), we do not find its performance to degrade severely. This encouraging result motivates us to extend our theoretical analysis to the mixed type and missing data settings in future work.

\begin{ack}
We would like to thank Vasilis Syrgkanis for insightful  discussions regarding causal inference methods and EconML usage; we thank Yordan Zaykov for engineering support; we thank Biwei Huang and Ruibo Tu for feedback that improved this manuscript; we thank Maria Defante, Karen Fassio, Steve Thomas and Dan Truax for insightful discussions on real-world needs which inspired the whole project. 
\end{ack}

\bibliography{bibliography}
\bibliographystyle{plainnat}

\clearpage
\appendix
\onecolumn

\section{Theoretical Considerations for DECI}
\label{app: FCause Consistency}

DECI can be categorized as a functional score-based causal discovery approach, which aims to find the model parameters $\theta$ and mean-field posterior $q_\phi(G)$ by maximizing the ELBO (\cref{eq:DECIELBO}). A key statistical property of DECI is whether it is capable of recovering the ground truth data generating distribution and true graph $G^0$ when DECI is \textbf{correctly specified} and with infinite data. In the following, we will show that DECI is indeed capable of this under standard assumptions. The main idea is to first show that the maximum likelihood estimate (MLE) recovers the ground truth due to the correctly specified model. Then, we prove that optimal solutions from maximizing the ELBO are closely related to the MLE under mild assumptions. 

\subsection{Notation and Assumptions}
First, we define the notation and assumptions required for the proof. We denote a random variable $\rvx\in\mathbb{R}^D$ with a ground truth data generating distribution $p(\rvx;{G}^0)$, where ${G}^0$ is a binary adjacency matrix representing the true causal DAG. DECI uses the additive noise model (ANM), defining the structural assignment $x_j = f(\rvx_{\Pa{j;{G}}};\theta)+z_j$, where $\Pa{j;G}$ are the parents of node $j$ specified by the adjacency matrix ${G}$ and $z_j$ are mutually independent noise variables with a joint distribution $p_\theta(z_1,\ldots,z_D)$. The mean-field variational distribution $q_\phi(G)$ is a product of independent Bernoulli distribution, and $p(G)$ is the soft prior over the graph defined by \cref{eq: soft prior}.

\begin{assumption}[Minimality]
For a distribution $p_\theta(\rvx;{G})$ generated by DECI with graph ${G}$ and parameter $\theta$, we assume the minimality condition holds \citep{spirtes2013directed}. Namely, the distribution $p_\theta(\rvx;{G})$ does not satisfy the local Markov condition with respect to any sub-graph of ${G}$. 
\end{assumption}
To satisfy this assumption in practice, one can leverage Proposition 17 from \citet{peters2014causal}, stating that the minimality condition can be satisfied if the the model is a continuous additive noise model (ANM) and its structural assignments are not a constant with respect to any of its arguments. In practice, one can always add an edge pruning step to remove spurious edges \citep{buhlmann2014cam, lachapelle2019gradient}. 

\begin{assumption}[DECI Structural Identifiability]
We assume that the DECI model satisfies the structural identifiability. Namely, for a distribution $p_\theta(\rvx;{G})$, the graph ${G}$ is said to be structural identifiable from $p_\theta(\rvx;{G})$ if there exists no other distribution $p_{\theta'}(\rvx;{G}')$ such that ${G}\neq{G}'$ and $p_\theta(\rvx;G)=p_{\theta'}(\rvx;G')$.
\end{assumption}
For general SEM, this assumption does not hold. In fact, one can always search for functions resulting in the independence of cause and mechanisms in both directions \citep{peters2017elements,zhang2015estimation}. However, by correctly restricting the function family and the form of structural assignments, one can obtain structural identifiability \citep{Hoyer2008discovery,peters2014causal, peters2014identifiability,peters2010identifying, shimizu2006linear,zhang2012identifiability}. From the formulation of DECI, it is a special case of the non-linear ANM \citep{Hoyer2008discovery,peters2014causal}. Given the non-linear ANM assumption, together with the minimality condition and some additional mild assumptions, Theorem 20 from \citet{peters2010identifying} proves that our DECI model is structural identifiable.
\begin{assumption}[Correctly Specified Model]
We assume the DECI model is correctly specified. Namely, there exists a parameter $\theta^*$ such that $p_{\theta^*}(\rvx;G^0)=p(\rvx;G^0)$.
\end{assumption}
In practice, this assumption is hard to check in general. However, we can leverage the universal approximation capacity of neural networks \citep{hornik1989multilayer}, meaning that they can approximate continuous functions arbitrarily well. This flexibility gives us a higher chance that this assumption indeed holds.
\begin{assumption}[Causal Sufficiency]
We assume DECI and the ground truth are causally sufficient. Namely, there are no latent confounders in the model. 
\end{assumption}
\begin{assumption}[Regularity of log likelihood]
We assume for all parameters $\theta$ and possible graphs $G$, the following holds:
\[
\mathbb{E}_{p(\rvx;G^0)}\left[\vert\log p_\theta(\rvx;G)\vert\right]<\infty
.\]
\end{assumption}

\subsection{MLE Recovers Ground Truth}
The likelihood has often been used as the score function for causal discovery. For example, \emph{Carefl}  \citep{khemakhem2021causal} adopts the likelihood ratio test \citep{hyvarinen2013pairwise} in the bivariate case, which is equivalent to selecting the causal directions with the maximized likelihood. However, they did not explicitly show that the resulting model recovers the ground truth for the multivariate case. In addition, \citet{zhang2015estimation} proved that maximizing likelihood for bivariate causal discovery is equivalent to minimizing the dependence between the cause and the noise variable. With the correctly specified, structural identifiable model, the resulting noise and cause are independent through maximizing the likelihood, indicating the graph is indeed causal. However, it is non-trivial to generalize this to the multivariate case that we treat in DECI. In the following, we will show that under a correctly specified model and with maximum likelihood training with infinite data, DECI can recover the unique ground truth graph $G^*=G^0$ and the true data generating distribution $p_{\theta^*}(\rvx;G^*)=p(\rvx;G^0)$, where $(\theta^*,G^*)$ are MLE solutions. 
\begin{proposition}
Assuming assumptions 1--5 hold, we denote $(\theta^*,G^*)$ as the MLE solution with infinite training data. Then, we have 
\[
p_{\theta^*}(\rvx;G^*) = p(\rvx;G^0)
\]
In particular, we have $G^* = G^0$.
\label{prop: MLE recovers ground truth}
\end{proposition}
\begin{proof}
The key idea is to show that with arbitrary $(\theta,G)$, we have the following:
\[
\lim_{N\rightarrow \infty}\frac{1}{N}\sum_{i=1}^N\log p_{\theta}(\rvx_i;G)\leq \lim_{N\rightarrow \infty}\frac{1}{N}\sum_{i=1}^N\log p(\rvx_i;G^0)
\]
By law of large numbers, we have
\[
\lim_{N\rightarrow \infty} \frac{1}{N}\sum_{i=1}^N\log p_{\theta}(\rvx_i;G)=\mathbb{E}_{p(\rvx;G^0)}\left[\log p_{\theta}(\rvx;G)\right].
\]
Then, we can show
\begin{align*}
    &\mathbb{E}_{p(\rvx;G^0)}\left[\log p_{\theta}(\rvx;G)\right]-\mathbb{E}_{p(\rvx;G^0)}\left[\log p(\rvx;G^0)\right]\\
    =&\mathbb{E}_{p(\rvx;G^0)}\left[\log \frac{p_\theta(\rvx;G)}{p(\rvx;G^0)}\right]\\
    \leq&\mathbb{E}_{p(\rvx;G^0)}\left[\frac{p_\theta(\rvx;G)}{p(\rvx;G^0)}-1\right]=\int p_\theta(\rvx;G)d\rvx-1=0
\end{align*}
where the inequality is due to $\log t\leq t-1$. 
With assumption 3--4, we know there are no latent confounders and the model is correctly specified. Then, the above equality holds when $(\theta^*,G^*)$ induces the same join likelihood $p(\rvx;G^0)$. Since the model is structural identifiable, we must have $G^*=G^0$. 
\end{proof}

\subsection{DECI Recovers the Ground Truth}
To show that DECI can indeed recover the ground truth by maximizing the ELBO, we first introduce an important lemma showing the KL regularizer $\mathrm{KL}[q_\phi(G)\Vert p(G)]$ is negligible in the infinite data limit. 
\begin{lemma}\label{lemma: prior negeligible}
Assume a variational distribution $q_\phi(G)$ over a space of graphs $\mathcal{G}_\phi$, where each graph $G\in \mathcal{G}_\phi$ has a non-zero associated weight $w_\phi(G)$. With the soft prior $p(G)$ defined as \cref{eq: soft prior} and bounded $\lambda,\rho,\alpha$, we have
\begin{equation}
    \lim_{N\rightarrow \infty}\frac{1}{N}\mathrm{KL}[q_\phi(G)\Vert p(G)] = 0.
\end{equation}
\end{lemma}
\begin{proof}
First, we write down the definition of KL divergence
\begin{equation*}
    KL[q_\phi(G)\Vert p(G)]
    =\sum_{G\in \mathcal{G}_\phi}w_\phi(G)\left[\log w_\phi(G)+\lambda \Vert G\Vert_F^2+\rho h(G)^2+\alpha h(G)
    +\log Z\right]
\end{equation*}
where $Z$ is the normalizing constant for the soft prior. From the definition and assumptions, it is trivial to know that $\log w_\phi(G)$, $\lambda\Vert G\Vert_F^2$ are bounded for all $G\in \mathcal{G}_\phi$. In the following, we show that $h(G)$ and $\log Z$ are also bounded. 

From the definition of the DAG penalty, we have $h(G)=tr(\exp(G\odot G))-D$. The matrix exponential is defined as 
\begin{align*}
    \mathrm{tr}(\exp(G\odot G))
    &=\sum_{k=0}^\infty \frac{1}{k!}\mathrm{tr}((G\odot G)^k)\\
    &=\sum_{k=0}^\infty \frac{1}{k!}\mathrm{tr}((G)^k)\\
    &=\sum_{k=0}^D \frac{1}{k!}tr((G)^k)
\end{align*}
where the second equality is due to the fact that $G$ is a binary adjacency matrix. From \citet{zheng2018notears}, we know that $\mathrm{tr}(G^k)$ counts for the number of closed loops with length $k$. Since the graph has finite number of nodes, the longest possible closed loop is $D$, resulting in the third equality. 

Thus, it is obvious that for any $k$, the number of closed loops with length $k$ must be finite. Hence, it is trivial that $h(G)<\infty$. 
Therefore, with bounded $\lambda,\rho,\alpha$, the un-normalized soft prior 
\[
|\exp(-\lambda\Vert G\Vert_F^2-\rho h(G)^2-\alpha h(G))|<\infty.
\]
Thus, the normalizing constant $Z$ must be finite since there are only finite number of possible graphs. 

Therefore, these must exists a constant $M_{\phi,G}$ such that $\log w_\phi(G)+\lambda \Vert G\Vert_F^2+\rho h(G)^2+\alpha h(G)+\log Z<M_{\phi,G}$.
Hence, we have
\begin{equation*}
    0\leq \mathrm{KL}[q_\phi(G)\Vert p(G)]
    <\sum_{G\in \mathcal{G}_\phi}w_\phi(G)M_{G,\phi}\leq \sqrt{\sum_{G\in\mathcal{G}_\phi}w^2_\phi(G)}\sqrt{\sum_{G\in\mathcal{G}_\phi}M^2_{G,\phi}}
    <\infty.
\end{equation*}
Thus, we have 
\[
\lim_{N\rightarrow \infty}\frac{1}{N}\mathrm{KL}[q_\phi(G)\Vert p(G)] = 0
\]
where the third inequality is obtained by using Cauchy-Schwarz inequality. 
\end{proof}

\newtheorem*{T1}{Theorem~\ref{thm: consistency of DECI}}
Now, we can prove that DECI can recover the ground truth. Recalling \Cref{thm: consistency of DECI},
\begin{T1}[DECI recovers the true distribution]
Assuming assumptions 1--5 are satisfied, the solution $(\theta',q'_\phi(G))$ from maximizing ELBO (\cref{eq:DECIELBO}) in the infinite data limit satisfies $q'_\phi(G)=\delta(G=G')$ where $G'$ is a unique graph. In particular, we have $G'=G^0$ and $p_{\theta'}(\rvx;G')=p(\rvx;G^0)$. 
\end{T1}
\begin{proof}
In terms of optimization, it is equivalent to re-write the ELBO (\cref{eq:DECIELBO}) as
\[
\frac{1}{N}\mathbb{E}_{q_\phi}\left[\log p_\theta(\rvx_1,\ldots,\rvx_N)\right]-\frac{1}{N}KL\left[q_\phi(G)\Vert p(G)\right].
\]
Now, under the infinite data limit and the definition of $q_\phi$, we have
\begin{align*}
&\lim_{N\rightarrow \infty}\frac{1}{N}\mathbb{E}_{q_\phi}\left[\log p_\theta(\rvx_1,\ldots,\rvx_N)\right]-\frac{1}{N}KL\left[q_\phi(G)\Vert p(G)\right]\\
    =&\lim_{N\rightarrow\infty}\frac{1}{N}\sum_{G\in\mathcal{G}_\phi}w_\phi(G)\log p_\theta(\rvx_1,\ldots,\rvx_N|G)-\frac{1}{N}KL[q_\phi(G)\Vert p(G)]\\
    =&\lim_{N\rightarrow\infty}\frac{1}{N}\sum_{i=1}^N\sum_{G\in\mathcal{G}_\phi}w_\phi(G)\log p_\theta(\rvx_i|G)\\
    =&\int p(\rvx;G^0)\sum_{G\in \mathcal{G}_\phi}w_\phi(G)\log p_\theta(\rvx|G)d\rvx,
\end{align*}
where the second and third equalities are from \Cref{lemma: prior negeligible} and the law of large numbers, respectively. 
Let $(\theta^*,G^*)$ be the solutions from MLE (\Cref{prop: MLE recovers ground truth}). Then, since $\sum_{G\in\mathcal{G}_\phi} w_\phi(G)=1$, $w_\phi(G)>0$, we have
\begin{align*}
    \sum_{G\in\mathcal{G}_\phi}w_\phi(G)\mathbb{E}_{p(\rvx;G^0)}\left[\log p_\theta(\rvx|G)\right]\leq \mathbb{E}_{p(\rvx;G^0)}\left[\log p_{\theta^*}(\rvx;G^*)\right]
\end{align*}
with the equality holding when every graph $G\in\mathcal{G}_\phi$ and associated parameter $\theta_G$ satisfies 
\begin{equation}
    \mathbb{E}_{p(\rvx;G^0)}\left[\log p_{\theta_G}(\rvx|G)\right] = \mathbb{E}_{p(\rvx;G^0)}\left[\log p_{\theta^*}(\rvx|G^*)\right].
\label{eq:1}
\end{equation}
From proposition \ref{prop: MLE recovers ground truth}, under correctly specified model, we have
\[
\mathbb{E}_{p(\rvx;G^0)}\left[\log p_{\theta^*}(\rvx|G^*)\right] = \mathbb{E}_{p(\rvx;G^0)}\left[\log p(\rvx;G^0)\right]
\]
Thus, for a $G'\in\mathcal{G}_\phi$ and associated parameter $\theta'$, the condition in \cref{eq:1} becomes
\begin{align*}
    &\mathbb{E}_{p(\rvx;G^0)}\left[\log p_{\theta'}(\rvx|G')\right] = \mathbb{E}_{p(\rvx;G^0)}\left[\log p(\rvx|G^0)\right]\\
    \Longrightarrow& \mathbb{E}_{p(\rvx;G^0)}\left[\log \frac{p_{\theta'}(\rvx;G')}{p(\rvx;G^0)}\right]=0\\
    \Longrightarrow& KL[p(\rvx;G^0)\Vert p_{\theta'}(\rvx;G')]=0,
\end{align*}
which implies $p_{\theta'}(\rvx;G')=p(\rvx;G^0)$. Since DECI is structural identifiable, this means $G' = G^0$ and it is unique. Thus, the graph space $\mathcal{G}_\phi$ only contains one graph $G'$, and $q'_\phi(G)=\delta(G=G')$.
\end{proof}
One should note that we do not explicitly restrict the noise distribution, indicating it still holds with the spline noise (\Cref{sec:deciforthewin}). However, the above theorem implicitly assumes that DECI is a special case of ANM for structural indentifiability and that the data has no missing values. Thus, it is not applicable for DECI with the mixed-type and missing value extensions. We leave a more general theoretical guarantee to future work.

\section{Additional Details for DECI} \label{sec:DECIdets}

\subsection{Optimization Details for Causal Discovery} \label{app:algorithm}

As mentioned in the main text, we gradually increase the values of $\rho$ and $\alpha$ as optimization proceeds, so that non-DAGs are heavily penalized. Inspired by \notears, we do this with a method that resembles the updates used by the augmented Lagrangian procedure for optimization \citep{nemirovsky1999optimization}. The optimization process interleaves two steps: (i) Optimize the objective for fixed values of $\rho$ and $\alpha$ for a certain number of steps; and (ii) Update the values of the penalty parameters $\rho$ and $\alpha$. The whole optimization process involves running the sequence (i)--(ii) until convergence, or until the maximum allowed number of optimization steps is reached.

\textbf{Step (i).} Optimizing the objective for some fixed values of $\rho$ and $\alpha$ using Adam \citep{kingma2014adam}. We optimize the objective for a maximum of $6000$ steps or until convergence, whichever happens first (we stop early if the loss does not improve for $1500$ optimization steps. If so, we move to step (ii)). We use Adam, initialized with a step-size of $0.01$. During training, we reduce the step-size by a factor of $10$ if the training loss does not improve for $500$ steps. We do this a maximum of two times. If we reach the condition a third time, we do not decrease the step-size and assume optimization has converged, and move to step (ii).

\textbf{Iterating (i)--(ii).} We initialize $\rho=1$ and $\alpha=0$. At the beginning of step (i) we measure the DAG penalty $P_1 = \mathbb{E}_{q_\phi(G)} h(G)$. Then, we run step (i) as explained above. At the beginning of step (ii) we measure the DAG penalty again, $P_2 = \mathbb{E}_{q_\phi(G)} h(G)$. If $P_2 < 0.65 \, P_1$, we leave $\rho$ unchanged and update $\alpha \leftarrow \alpha + \rho \, P_2$. Otherwise, if $P_2 \geq 0.65 \, P_1$, we leave $\alpha$ unchanged and update $\rho \leftarrow 10\,\rho$. We repeat the sequence (i)--(ii) for a maximum of $100$ steps or until convergence (measured as $\alpha$ or $\rho$ reaching some max value which we set to $10^{13}$ for both), whichever happens first. %

\subsection{Other Hyperparameters.} 

We use $\lambda_s = 5$ in our prior over graphs \cref{eq: soft prior}.
For ELBO MC gradients we use the Gumbel softmax method with a hard forward pass and a soft backward pass with temperature of $0.25$.

The functions \cref{eq:fcause_functions} used in DECI's SEM, $\zeta$ and $ \ell$, are 2 hidden layer MLPs with 128 hidden units per hidden layer. These MLPs use residual connections and layer-norm at every hidden layer.

For the non-Gaussian noise model in \cref{eq:noise_model}, the bijection $\kappa$ is an 8 bin rational quadratic spline \citep{Durkan2019spline} with learnt parameters. 

In \cref{sec:deci_ate}, for ATE estimation we compute expectations by drawing 1000 graphs from DECI's graph posterior $q_\phi$ and for each graph we draw 2 samples of $\rvx_{Y}$ for a total of 2000 samples.
For CATE estimation, we need to train a separate surrogate predictor per graph samples. We draw 10 different graph samples and 10000 $(\rvx_{C}, \rvx_{Y})$ pair samples for each graph. We use these to train the surrogate models.

Our surrogate predictor is a basis function linear model with 3000 random Fourier features drawn such that the model approximates a Gaussian process with a radial basis function kernel of lengthscale equal to $1$ \citep{Felix2016orthogonal}.

\subsection{ELBO Derivation} \label{sec:elboderivation}

The goal of maximum likelihood involves maximizing the likelihood of the observed variables. For DECI (with fully observed datasets) this corresponds to the log-marginal likelihood
\begin{equation}
    \log p_\theta(x^1, \hdots, x^N) = \log \sum_A p(G) \prod_n p_\theta(x^n|G). \label{eq:appml}
\end{equation}
Marginalising $G$ in the equation above is intractable, even for moderately low dimensions, since the number of terms in the sum grows exponentially with the size of $G$ (which grows quadratically with the data dimensionality $D$).

Variational inference proposes to use a distribution $q_\phi(G)$ to build the ELBO, a lower bound of the objective from \cref{eq:appml}, as follows:
\begin{align}
    \log p_\theta(x^1, \hdots, x^N) & = \log \sum_G p(G) \prod_n p_\theta(x^n|G)\\
    & = \log \sum_G q_\phi(G) \frac{p(G) \prod_n p_\theta(x^n|G)}{q_\phi(G)}\\
    & = \log \mathbb{E}_{q_{\phi}(G)} \left[\frac{p(G) \prod_n p_\theta(x^n|G)}{q_\phi(G)} \right]\\
    & \geq \mathbb{E}_{q_{\phi}(G)} \left[\log \frac{p(G) \prod_n p_\theta(x^n|G)}{q_\phi(G)} \right] & \mbox{(Jensen's inequality)}\\
    & = \mathbb{E}_{q_{\phi}(G)} \left[\log p(G) \prod_n p_\theta(x^n|G) \right] + H(q_\phi)\\
    & = \mathrm{ELBO}(\phi, \theta),
\end{align}

where we denote $H(q_\phi) = -\mathbb{E}_{q_\phi(G)} \log q_\phi(G)$ for the entropy of the distribution $q_\phi$. Interestingly, the distribution $q_\phi$ that maximizes the ELBO is exactly the one that minimizes the KL-divergence between the approximation and the true posterior, $\mathrm{KL}(q_\phi(G) \Vert p_\theta(G\vert x^1\hdots, x^N))$ (see, e.g.~\citet{blei2017variational}). This is why $q_\phi$ can be used as a posterior approximation.

\subsection{Intervened Density Estimation with DECI}

Apart from (C)ATE estimation, DECI may also be used to evaluate densities under intervened distributions. For a given graph, the density of some observation vector $\rva$ is computed by evaluating the base distribution density after inverting the SEM
\begin{gather}
    p_{\theta}(\rvx = \rva|G^{m}) = \prod_{i} p(\rvz_{i}= (\rva_{i} - f_{i}(\rva_{\Pa{i; G^{m}}}))) 
\end{gather}
noting that the transformation Jacobian is the identity. We then marginalise the graphs using Monte Carlo:
\begin{gather}\label{eq:MC_graph_marginalisation}
    p_{\theta}(\rvx = \rva ) \approx \frac{1}{M}\sum_{m}^{M} p_{\theta}(\rvx= \rva | G^{m}); \quad G^{m} \sim q_{\phi}(G).
\end{gather}
In the rest of this section we derive methods that allow using DECI to estimate causal quantities.

Under $G_{\Do{\rvx_{T}}}$, $i \in T$ correspond to parent nodes and we have the following factorisation: $p(\rvx | G_{\Do{\rvx_{T}}}) = p(\rvx_{\setminus{T}} | G_{\Do{\rvx_{T}}}) \prod_{i \in T} p(\rvx_{i})$. We can then evaluate the interventional density of an observation $\rvx_{\setminus T}=\rva$ with DECI as
\begin{align}
     p_{\theta}&(\rvx_{\setminus T}=\rva | \text{do}(\rvx_{T}=\rvb), G^{m}) \notag \\ &= \frac{p_{\theta}(\rvx_{\setminus T}=\rva, \rvx_{T}=\rvb | G^{m}_{\Do{\rvx_{T}}})}{p_{\theta}(\rvx_{T}=\rvb | G^{m}_{\Do{\rvx_{T}}} )} \notag \\
     &= \frac{p_{\theta}(\rvx_{\setminus T}=\rva | \rvx_{T}=\rvb, G^{m}_{\Do{\rvx_{T}}})p_{\theta}(\rvx_{T}=\rvb )}{p_{\theta}(\rvx_{T}=\rvb )} \notag \\
     &= \prod_{j \in \setminus{T}} p(\rvz_{i}= (\rva_{i} - f_{i}(\rva_{\Pa{i; G^{m}_{\Do{\rvx_{T}}}}}))),
\end{align}
which amounts to evaluating the density of the exogenous noise correspondint to non-intervened variables. We can then marginalise the graph using Monte Carlo as in \cref{eq:MC_graph_marginalisation}.

\subsection{Relationship with \citet{khemakhem2021causal}}
\label{app:causalflow}

\citet{khemakhem2021causal} introduced \carefl, a method that uses autoregressive flows \citep{ huang2018neural, kingma2016improved} to learn causal-aware models, using the variables' causal ordering to define the autoregressive transformations. The method's main benefit is its ability to model complex nonlinear relationships between variables. However, \carefl\ alone is insufficient for causal discovery, as it requires the causal graph structure as an input.
 The authors propose a two-step approach. First, run a traditional constraint-based method (e.g., PC) to find the graph's skeleton and orient as many edges as possible, and second, fit several flow models to determine the orientation of the remaining edges. The drawbacks of this approach include the dependence on an external causal discovery methods (which will inherently limit \carefl's performance to that of the method used), and the cost of fitting multiple flow models to orient the edges that are left unoriented after the first step. 
 Our method extends \citet{khemakhem2021causal} to learn the causal graph among multiple variables and perform end-to-end causal inference. 

\subsection{Discussion on Causal Discovery Methods }
\label{app:generalECI_discussion}
When performing causal discovery, DECI returns a posterior over graphs. Most other causal discovery methods return either a single graph or an equivalence class of graphs. However, we can re-cast these methods in the probabilistic framework used by DECI by noting that a posterior over graphs takes the form
\begin{gather}\label{eq:graph_posterior}
    p(G|\rmX) = \frac{p(\rmX | G) p(G)}{\sum_{G} p(\rmX | G) p(G)}.
\end{gather}
In this equation, the likelihood measures the degree of compatibility of a certain DAG architecture with the observed data. For score-based discovery methods \citep{chickering2002optimal,chickering2015selective, chickering2020statistically,huang2018generalized} we take the score to be $\log p(\rmX | G)$. For functional discovery methods \citep{hoyer2008nonlinear, shimizu2006linear, zhang2009identifiability} we use the exogenous variable log-density. Constraint-based methods \citep{spirtes1991algorithm, spirtes2000causation} can also be cast in this light by assuming a uniform distribution over all graphs in their outputted equivalence class $\mathcal{G}$: $\log p(\rmX | G) = -\log |\mathcal{G}|,\forall \,G \in \mathcal{G}$. 
To what degree these methods succeed at constraining the space of possible graphs will depend on how well their respective assumptions are met and the amount of data available \citep{Hoyer2008discovery}.

\section{Unified View of Causal Discovery Methods} \label{app:unified_flow}

This section introduces a simple analysis showing that, similarly to DECI, most causal discovery methods based on continuous optimization can be framed from a probabilistic perspective as fitting a flow. The benefits of this unified perspective are twofold. First, it allows a simple comparison between methods, shedding light on the different assumptions used by each one, their benefits and drawbacks. Second, it simplifies the development of new tools to improve these methods, since any improvements to one of them can be easily mapped to the others by framing them in this unified framework (e.g. our extensions to handle missing values and flexible noise distributions can be easily integrated with \notears).

The connection between causal discovery methods based on continuous optimization and flow-based models uses the concept of a weighted adjacency matrix $\wt\in \mathbb{R}^{D\times D}$ linked to a function $f(\rvx; \theta): \mathbb{R}^D \rightarrow \mathbb{R}^D$. Loosely speaking, these matrices can be seen as characterizing how likely is each output of $f(\rvx; \theta)$ to depend on each component of the input $\rvx$. For instance, $\wt_{j, i} = 0$ indicates that $f_i(\rvx;\theta)$ is completely independent of $x_j$. Such adjacency matrices can be constructed efficiently for a wide range of parameterizations for $f$, such as multi layer perceptrons and weighted combinations of nonlinear functions. We refer the reader to \citet{zheng2020learning} for details.

\begin{lemma} \label{lemma:flowobj}
Let $f(\rvx; \theta): \mathbb{R}^D \rightarrow \mathbb{R}^D$ be a $\theta$-parameterized function with weighted adjacency matrix $\wt\in \mathbb{R}^{D\times D}$. Given a dataset $\{\rvx^1, \hdots, \rvx^N\}$, fitting a flow with the transformation $\rvz = \rvx - f(\rvx; \theta)$, base distribution $p_\rvz$ and a hard acyclicity constraint on $\wt$ is equivalent to solving
\begin{equation}
\max_\theta \sum_{n=1}^N \log p_\rvz(\rvx^n - f(\rvx^n;\theta)) \quad \mathrm{s.t.} \quad h(\wt) = 0, \label{eq:flowobj}
\end{equation}
where $h(\cdot)$ is the algebraic characterization of DAGs from \cref{eq:dagness}.
\end{lemma}

\begin{proof}
The acyclicity constraint is enforced by constraining the optimization domain to $\Theta = \{\theta: h(\wt) = 0\}$. Then, the maximum likelihood objective can be written as
\begin{align}
    \sum_n \log p_\theta(\rvx^n) & = \sum_n \log p_z(\rvx^n - f(\rvx^n; \theta)) + \log \left|\det \frac{\mathrm{d}  (\rvx^n - f(\rvx^n;\theta))}{\mathrm{d} \rvx^n}\right| \label{eq:prop1}\\
    & = \sum_n \log p_z(\rvx^n -  f(\rvx^n;\theta)), 
    \label{eq:prop2}
\end{align}
where the first equality we use the change of variable formula, valid because the transformation $\rvz = g(\rvx;\theta) = \rvx-f(\rvx;\theta)$ is invertible for any $\theta\in \Theta$, and the second equality uses that the function $f(\rvx^n;\theta)$ has Jacobian-determinant equal to $1$, due to the constraint $\Theta = \{\theta: h(\wt) = 0\}$.
\end{proof}

\Cref{lemma:flowobj} is the main building block in the formulation of continuous optimization-based causal discovery methods from a probabilistic perspective as fitting flow models. This is simply because the objective used by each of these methods can be exactly recovered from \cref{eq:flowobj} with specific choices for $f(\rvx; \theta)$ and $p_\rvz$.
\begin{description}[topsep=-3pt]
    \setlength\itemsep{0em}
	\item[\notears\ \citep{zheng2018dags}] uses a standard isotropic Gaussian for $p_\rvz$ and a linear transformation for $f(\rvx, \theta)$. (This is similar to DECI-Gaussian, although DECI permits fully nonlinear functions.)
    
	\item[\notearsmlp\ \citep{zheng2020learning}] uses a standard isotropic Gaussian for $p_\rvz$ and $D$ independent multi-layer perceptrons, one for each component of $f(\rvx, \theta)$.

	\item[\notearssob\ \citep{zheng2020learning}] uses a standard isotropic Gaussian for $p_\rvz$ and a weighted linear combination of nonlinear basis functions.
	
	\item[\textit{GAE} \citep{ng2019graph}] uses a standard isotropic Gaussian for $p_\rvz$ and a GNN for $f(\rvx, \theta)$.
	
	\item[\grandag\ \citep{lachapelle2019gradient}] uses a factorized Gaussian with mean zero and learnable scales for $p_\rvz$ and $D$ multi layer perceptrons, one for each component of $f(\rvx, \theta)$.
    
    \item[\golem\ \citep{ng2020role}.] This is a linear method whose original formulation was already in a probabilistic perspective, using a linear transformation for $f(\rvx; \theta)$.
\end{description}

In summary, recently proposed causal discovery methods based on continuous optimization can be formulated from a probabilistic perspective as fitting a flow with different constraints, transformations, and base distributions. This unified formulation sheds light on the assumptions done by each method (e.g. a Gaussian noise assumption, either implicitly as in \notears\ or explicitly as in \grandag) and, more importantly, simplifies the development of new tools to improve them. For instance, the ideas proposed to deal with partially-observed datasets and non-Gaussian noise are readily applicable to any of the causal discovery methods mentioned in this section, addressing some of their limitations \cite{kaiser2021unsuitability, loh2014high,  Reisach2021beware}.

\section{Datasets Details} \label{app:benchmark} 

Our two benchmark datasets are constructed following similar procedures described in \citet{louizos2017causal}.

\textbf{IHDP \citep{hill2011bayesian}.} This dataset contains measurements of both infants (birth weight, head circumference, etc.) and their mother (smoked cigarettes, drank alcohol, took drugs, etc) during real-life data collected in a randomized experiment. The main task is to estimate the effect of home visits by specialists on future cognitive test scores of infants. The outcomes of treatments are simulated artificially as in \citep{hill2011bayesian}; hence the outcomes of both treatments (home visits or not) on each subject are known. Note that for each subject, our models are only exposed to only one of the treatments; the outcomes of the other potential/counterfactual outcomes are hidden from the mode, and are only used for the purpose of ATE/CATE evaluation. To make the task more challenging, additional confoundings are manually introduced by removing a subset (non-white mothers) of the treated children population. In this way we can construct the IHDP dataset of 747 individuals with 6 continuous covariates and 19 binary covariates. We use 10 replicates of different simulations based on setting B (log-linear response surfaces) of \citep{hill2011bayesian}, which can downloaded from \url{https://github.com/AMLab-Amsterdam/CEVAE}. We use a 70\%/30\% train-test split ratio.  Before training our models, all continuous covariates are normalized. 

\textbf{TWINS \citep{almond2005costs}.} This dataset consists of twin births in the US between 1989-1991. Only twins which with the same sex born weighing less than 2kg are considered. The treatment is defined as being born as the heavier one in each twins pair, and the outcome is defined as the mortality of each twins in their first year of life. Therefore, by definition for each pair of twins, we can observe the outcomes of both treatments (the lighter twin and heavier twin). However, during training, only one of the treatment is visible to our models, and the other potential outcome is unknown to the model and are only used for evaluation. The raw dataset is downloaded from \url{https://github.com/AMLab-Amsterdam/CEVAE}. Following \citet{louizos2017causal}, we also introduce artificial confounding using the categorical \texttt{GESTAT10} variable. This is done by assigning treatments (factuals) using the conditional probability $t_i|\mathbf{x}_i, z_i = \mathrm{Bern}(\sigma(w_0^T\mathbf{x}_i + w_h (z_i/10 - 0.1)))$, where $t_i$ is the treatment assignment for subject $i$, $z_i$ is the corresponding \texttt{GESTAT10} covariate, $\mathbf{x}_i$ denotes the other remaining covariates. Both $w_0$ and $w_h$ are randomly generated as $w_0 \sim 
\mathcal{N}(0, 0.1*I)$, $w_h \sim \mathcal{N}(5,0.1)$. All continuous covariates are normalized.  

\textbf{Ground Truth ATE and CATE Estimation for TWINS and IHDP.}  In both benchmark datasets, since the held-out hypothetical outcomes of counterfactual treatments are already known, the the ground truth ATE can be naively estimated by averaging the difference between the factual and counterfactual outcomes across the entire dataset. The CATE estimation is a bit tricky, since both datasets contains covariates collected from real-world experiments, in which the underlying ground truth causal graph structure is unknown. As a result, exact CATE estimation is generally impossible for continuous conditioning sets. Therefore, when evaluating the CATE estimation performance on \textbf{TWINS} and \textbf{IHDP}, we focus only on discrete variables (binary and categorical) as conditioning set. This allows unbiased estimation of ground truth CATE by simply averaging the treatment effects on subgroups of subjects in the dataset, that have the corresponding discrete value in the conditioning set. We consider only single conditioning variable at a time, and estimate the corresponding CATE for evaluation.

\subsection{CSuite}
\label{sec:app:csuite}

We develop Causal Suite (CSuite), a number of small to medium (2--12 nodes) synthetic datasets generated from hand-crafted Bayesian networks with the intention of testing different capabilities of causal discovery and inference methods. All continuous-only datasets take the form of additive noise models.

Each dataset comes with a training set of 2000 samples, and between 1 and 2 intervention test sets. Each intervention test set has a treatment variable, treatment value, reference treatment value and effect variable. We estimate the ground truth ATE by drawing 2000 samples from the treated and reference intervened distributions. For the datasets used to evaluate CATE, we generate samples from \emph{conditional} intervened distributions by using Hamiltonian Monte Carlo. We employ a burn-in of 10k steps and a thinning factor of 5 to generate 2000 conditional samples, which we then use to compute our ground truth CATE estimate. We note that because all ground truth causal quantities are estimated from samples, there is a lower bound on the expected error that can be obtained by our methods. When methods obtain an error equal or lower we say that they have solved the task.

\paragraph{lingauss} A two node graph (\Cref{fig:two_node_pgm}) with a linear relationship and Gaussian noise. We have $X_1\sim N(0,1)$ and $X_2 = \tfrac{1}{2}X_1 + \tfrac{\sqrt{3}}{2}Z_2$ where $Z_2 \sim N(0,1)$ is independent of $X_1$. The observational distribution is symmetrical in $X_1 \leftrightarrow X_2$. The graph is not identifiable. The best achievable performance on this dataset is obtained when there is a uniform distribution over edge direction.

\paragraph{linexp} A two node graph (\Cref{fig:two_node_pgm}) with a linear functional relationship, but with exponentially distributed additive noise.
We have $X_1\sim N(0,1)$ and $X_2 = \tfrac{1}{2}X_1 + \tfrac{\sqrt{3}}{2}(Z_2 - 1)$ where $Z_2 \sim \text{Exp}(1)$ is independent of $X_1$.
By using non-Gaussian noise, the graph becomes identifiable. However, the inference problem will be more challenging for methods sensitive to outliers, such as those that assume Gaussian noise.

\paragraph{nonlingauss} A two node graph (\Cref{fig:two_node_pgm}) with a nonlinear relationship and Gaussian additive noise. We have $X_1 \sim N(0,1)$ and $X_2 = \sqrt{6}\exp(-X_1^2) + \alpha Z_2$ where $Z_2 \sim N(0,1)$ is independent of $X_1$ and $\alpha^2 = {1 - 6\left(\frac{1}{\sqrt{5}} - \frac{1}{3}\right)}$. Note $\text{Var}(X_2)=1$ and $\text{Cov}(X_1,X_2)=0$. By having a linear correlation of zero between $X_1$ and $X_2$, this dataset creates a potential failure mode for causal inference methods that assume linearity.

\paragraph{nonlin\_simpson} A synthetic Simpson's paradox, using the graph \Cref{fig:nonlin_simpson_pgm}: if the confounding factor $X_3$ is not adjusted for, the relationship between the treatment $X_1$ and effect $X_2$ reverses.
The variable $X_4$ correlates strongly with the effect, but must not be used for adjustment.
Choosing an incorrect adjustment set when estimating $\E[X_2|\Do{X_1}]$ leads to a significantly incorrect ATE estimate.
All variables are continuous, with nonlinear structural equations and non-Gaussian additive noise.

\paragraph{symprod\_simpson} Another Simpson's paradox using the graph \Cref{fig:symprod_simpson_pgm}. This dataset is similar to nonlin\_simpson with 2 key differences: 1) the effect variable is the result of a product between the confounding variable and the treatment variable. This makes drawing causal inferences require non-linear function estimation. Additionally, the ATE is close to 0. The conditioning variable for the CATE task is a descendant of the confounding variable. 
This dataset probes for methods' capacity to reduce their uncertainty about a confounding variables based on values of its child variables.

\paragraph{large\_backdoor} A nine node graph, as shown in \Cref{fig:large_backdoor_pgm}. This dataset is constructed so that there are many possible choices of backdoor adjustment set.
While both minimal and maximal adjustment sets can result in a correct solution, the a minimal adjustment set results in a much lower-dimensional adjustment problem and thus will result in lower variance solutions. The conditioning node for the CATE task is a child of the root variable. Thus the CATE task probes for methods' capacity to infer the value of an observed confounder from one of its children.
All variables are continuous, with nonlinear structural equations and non-Gaussian additive noise.

\paragraph{weak\_arrows} A nine node graph, as shown in \Cref{fig:weak_arrows_pgm}. Unlike the previous dataset, when the true graph is known, a large adjustment set must be used. The causal discovery challenge revolves around finding all arrows, which are scaled to be relatively weak, but which have significant predictive power for $X_9$ in aggregate.
This dataset tests methods' capacity to identify the full adjustment set and adjust for a large number of variables simultaneously.

\paragraph{cat\_to\_cts} A two node (\Cref{fig:two_node_pgm}) graph with categorical $X_1$ and continuous $X_2$ with an additive noise model. We have $X_1 \sim \text{Cat}\left(\tfrac{1}{4}, \tfrac{1}{4}, \tfrac{1}{2}\right)$ takes values in $\{0,1,2\}$ and $X_2 = X_1 + \tfrac{8}{5}(s(Z_2) - 1)$ where $s(x) = \log(\exp(x) + 1)$ is the softplus function, and $Z_2 \sim N(0,1)$ is independent of $X_1$.

\paragraph{cts\_to\_cat} A two node (\Cref{fig:two_node_pgm}) graph with continuous $X_1$ and categorical $X_2$. 
We take $X_1 \sim U(-\sqrt{3},\sqrt{3})$ and $X_2$ categorical on $\{0,1,2\}$ with the following conditional probabilities
\begin{equation}
    p(X_2|X_1=x_1) = \begin{cases}
    \left(\tfrac{6}{13},\tfrac{6}{13},\tfrac{1}{13} \right) & \text{ if } x_1 < -\tfrac{\sqrt{3}}{3} \\
    \left(\tfrac{1}{8},\tfrac{3}{4},\tfrac{1}{8} \right) & \text{ if } -\tfrac{\sqrt{3}}{3} \le x_1 < \tfrac{\sqrt{3}}{3} \\
    \left(\tfrac{1}{3},\tfrac{1}{3},\tfrac{1}{3} \right) & \text{ if } x_1 > \tfrac{\sqrt{3}}{3} \\
    \end{cases}
\end{equation}
In this problem, we treat $X_2$ as the treatment and $X_1$ as the target, giving a theoretical ATE of zero.

\paragraph{mixed\_simpson} Similar to the nonlin\_simpson dataset, using the graph of \Cref{fig:nonlin_simpson_pgm}, but with $X_3$ categorical on three categories, and $X_1$ binary.

\paragraph{large\_backdoor\_binary\_t} Similar to the large\_backdoor dataset, using the graph of 
\Cref{fig:large_backdoor_pgm}, but with $X_8$ binary.

\paragraph{weak\_arrows\_binary\_t} Similar to the weak\_arrows dataset, using the graph of \Cref{fig:weak_arrows_pgm}, but with $X_8$ binary.

\paragraph{mixed\_confounding} A large, mixed type dataset with 12 variables, as shown in \Cref{fig:mixed_confounding_pgm}. In this dataset, $X_1,X_5$ are binary, $X_3, X_6,X_8$ are categorical on three categories, and other variables are continuous. We utilise nonlinear structural equations and non-Gaussian additive noise.

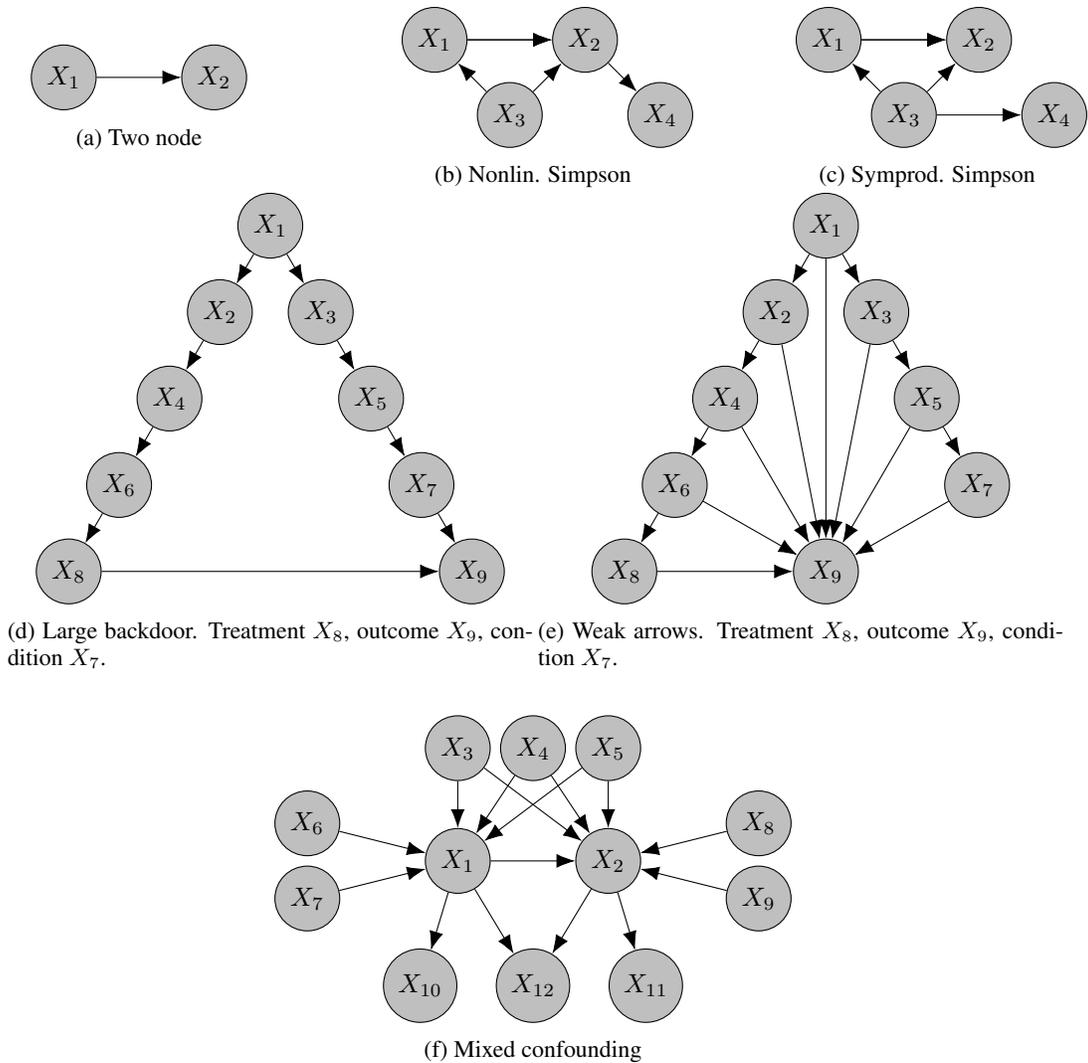
\begin{figure}
\begin{subfigure}{0.25\columnwidth}
    \centering
    \begin{tikzpicture}
	\node (y) [draw, circle, fill=lightgray] at (-1, 1) {$X_1$};
	\node (z) [draw, circle,fill=lightgray] at (1, 1) {$X_2$};
	\draw[-{Latex[length=2.5mm]}] (y) -> (z);
	\end{tikzpicture}
    \caption{Two node}\label{fig:two_node_pgm}
\end{subfigure}
\begin{subfigure}{0.25\columnwidth}
    \centering
    \begin{tikzpicture}
	\node (y) [draw, circle, fill=lightgray] at (-1, 1) {$X_1$};
	\node (z) [draw, circle,fill=lightgray] at (1, 1) {$X_2$};
	\draw[-{Latex[length=2.5mm]}] (y) -> (z);
	\node (a) [draw, circle, fill=lightgray] at (0, 0) {$X_3$};
	\node (b) [draw, circle,fill=lightgray] at (2, 0) {$X_4$};
	\draw[-{Latex[length=2.5mm]}] (y) -> (z);
	\draw[-{Latex[length=2.5mm]}] (a) -> (z);
	\draw[-{Latex[length=2.5mm]}] (a) -> (y);
	\draw[-{Latex[length=2.5mm]}] (z) -> (b);
	\end{tikzpicture}
    \caption{Nonlin. Simpson}\label{fig:nonlin_simpson_pgm}
\end{subfigure}
\begin{subfigure}{0.25\columnwidth}
    \centering
    \begin{tikzpicture}
	\node (y) [draw, circle, fill=lightgray] at (-1, 1) {$X_1$};
	\node (z) [draw, circle,fill=lightgray] at (1, 1) {$X_2$};
	\draw[-{Latex[length=2.5mm]}] (y) -> (z);
	\node (a) [draw, circle, fill=lightgray] at (0, 0) {$X_3$};
	\node (b) [draw, circle,fill=lightgray] at (2, 0) {$X_4$};
	\draw[-{Latex[length=2.5mm]}] (y) -> (z);
	\draw[-{Latex[length=2.5mm]}] (a) -> (z);
	\draw[-{Latex[length=2.5mm]}] (a) -> (y);
	\draw[-{Latex[length=2.5mm]}] (a) -> (b);
	\end{tikzpicture}
    \caption{Symprod. Simpson}\label{fig:symprod_simpson_pgm}
\end{subfigure}
\vspace{15pt}
\begin{subfigure}{0.5\columnwidth}
    \centering
    \begin{tikzpicture}
	\node (x1) [draw, circle, fill=lightgray] at (0,0) {$X_1$};
	\node (x2) [draw, circle, fill=lightgray] at (-.67, -1.15) {$X_2$};
	\node (x3) [draw, circle, fill=lightgray] at (.67, -1.15) {$X_3$};
	\node (x4) [draw, circle, fill=lightgray] at (-1.34, -2.3) {$X_4$};
	\node (x5) [draw, circle, fill=lightgray] at (1.34, -2.3) {$X_5$};
	\node (x6) [draw, circle, fill=lightgray] at (-2.01, -3.45) {$X_6$};
	\node (x7) [draw, circle, fill=lightgray] at (2.01, -3.45) {$X_7$};
	\node (x8) [draw, circle, fill=lightgray] at (-2.68, -4.6) {$X_8$};
	\node (x9) [draw, circle, fill=lightgray] at (2.68, -4.6) {$X_9$};
	\draw[-{Latex[length=2.5mm]}] (x1) -> (x2);
	\draw[-{Latex[length=2.5mm]}] (x1) -> (x3);
	\draw[-{Latex[length=2.5mm]}] (x2) -> (x4);
	\draw[-{Latex[length=2.5mm]}] (x3) -> (x5);
	\draw[-{Latex[length=2.5mm]}] (x4) -> (x6);
	\draw[-{Latex[length=2.5mm]}] (x5) -> (x7);
	\draw[-{Latex[length=2.5mm]}] (x6) -> (x8);
	\draw[-{Latex[length=2.5mm]}] (x7) -> (x9);
	\draw[-{Latex[length=2.5mm]}] (x8) -> (x9);
	\end{tikzpicture}
    \caption{Large backdoor. Treatment $X_8$, outcome $X_9$, condition $X_7$.}\label{fig:large_backdoor_pgm}
\end{subfigure}
\begin{subfigure}{0.5\columnwidth}
    \centering
    \begin{tikzpicture}
	\node (x1) [draw, circle, fill=lightgray] at (0,0) {$X_1$};
	\node (x2) [draw, circle, fill=lightgray] at (-.67, -1.15) {$X_2$};
	\node (x3) [draw, circle, fill=lightgray] at (.67, -1.15) {$X_3$};
	\node (x4) [draw, circle, fill=lightgray] at (-1.34, -2.3) {$X_4$};
	\node (x5) [draw, circle, fill=lightgray] at (1.34, -2.3) {$X_5$};
	\node (x6) [draw, circle, fill=lightgray] at (-2.01, -3.45) {$X_6$};
	\node (x7) [draw, circle, fill=lightgray] at (2.01, -3.45) {$X_7$};
	\node (x8) [draw, circle, fill=lightgray] at (-2.68, -4.6) {$X_8$};
	\node (x9) [draw, circle, fill=lightgray] at (0, -4.6) {$X_9$};
	\draw[-{Latex[length=2.5mm]}] (x1) -> (x2);
	\draw[-{Latex[length=2.5mm]}] (x1) -> (x3);
	\draw[-{Latex[length=2.5mm]}] (x2) -> (x4);
	\draw[-{Latex[length=2.5mm]}] (x3) -> (x5);
	\draw[-{Latex[length=2.5mm]}] (x4) -> (x6);
	\draw[-{Latex[length=2.5mm]}] (x5) -> (x7);
	\draw[-{Latex[length=2.5mm]}] (x6) -> (x8);
	\draw[-{Latex[length=2.5mm]}] (x7) -> (x9);
	\draw[-{Latex[length=2.5mm]}] (x8) -> (x9);
	\draw[-{Latex[length=2.5mm]}] (x1) -> (x9);
	\draw[-{Latex[length=2.5mm]}] (x2) -> (x9);
	\draw[-{Latex[length=2.5mm]}] (x3) -> (x9);
	\draw[-{Latex[length=2.5mm]}] (x4) -> (x9);
	\draw[-{Latex[length=2.5mm]}] (x5) -> (x9);
	\draw[-{Latex[length=2.5mm]}] (x6) -> (x9);
	\end{tikzpicture}
    \caption{Weak arrows. Treatment $X_8$, outcome $X_9$, condition $X_7$.}\label{fig:weak_arrows_pgm}
\end{subfigure}
\vspace{15pt}
\begin{subfigure}{\columnwidth}
    \centering
    \begin{tikzpicture}
	\node (x1) [draw, circle, fill=lightgray] at (-1, 1) {$X_1$};
	\node (x2) [draw, circle,fill=lightgray] at (1, 1) {$X_2$};
	\node (x3) [draw, circle,fill=lightgray] at (-1, 2.5) {$X_3$};
	\node (x4) [draw, circle,fill=lightgray] at (0, 2.5) {$X_4$};
	\node (x5) [draw, circle,fill=lightgray] at (1, 2.5) {$X_5$};
	\node (x6) [draw, circle,fill=lightgray] at (-3, 1.5) {$X_6$};
	\node (x7) [draw, circle,fill=lightgray] at (-3, .5) {$X_7$};
	\node (x8) [draw, circle,fill=lightgray] at (3, 1.5) {$X_8$};
	\node (x9) [draw, circle,fill=lightgray] at (3, .5) {$X_9$};
	\node (x10) [draw, circle,fill=lightgray] at (-1.5, -.66) {$X_{10}$};
	\node (x12) [draw, circle,fill=lightgray] at (0, -.66) {$X_{12}$};
	\node (x11) [draw, circle,fill=lightgray] at (1.5, -.66) {$X_{11}$};
	\draw[-{Latex[length=2.5mm]}] (x1) -> (x2);
	\draw[-{Latex[length=2.5mm]}] (x3) -> (x1);
	\draw[-{Latex[length=2.5mm]}] (x3) -> (x2);
	\draw[-{Latex[length=2.5mm]}] (x4) -> (x1);
	\draw[-{Latex[length=2.5mm]}] (x4) -> (x2);
	\draw[-{Latex[length=2.5mm]}] (x5) -> (x1);
	\draw[-{Latex[length=2.5mm]}] (x5) -> (x2);
	\draw[-{Latex[length=2.5mm]}] (x6) -> (x1);
	\draw[-{Latex[length=2.5mm]}] (x7) -> (x1);
	\draw[-{Latex[length=2.5mm]}] (x8) -> (x2);
	\draw[-{Latex[length=2.5mm]}] (x9) -> (x2);
	\draw[-{Latex[length=2.5mm]}] (x1) -> (x10);
	\draw[-{Latex[length=2.5mm]}] (x2) -> (x11);
	\draw[-{Latex[length=2.5mm]}] (x1) -> (x12);
	\draw[-{Latex[length=2.5mm]}] (x2) -> (x12);
	\end{tikzpicture}
    \caption{Mixed confounding}\label{fig:mixed_confounding_pgm}
\end{subfigure}
    \caption{CSuite graphs. Unless otherwise stated, we take $X_1$ as the treatment, $X_2$ as the outcome, and for CATE we take $X_3$ as the conditioning variable.}
    \label{fig:csuite_pgms}
\end{figure}

\section{Additional Results}
\label{sec:app:additional}

\subsection{Causal Discovery Results under Gaussian Exogenous Noise}

\Cref{fig:cd_summary} in the main text shows causal discovery results for the case where synthetic data was generated using non-Gaussian noise. In that case it was observed that using DECI together with a flexible noise model performed better than DECI with a Gaussian noise model. 
\Cref{fig:cd_summary_gauss} shows results for synthetic data generated using Gaussian noise. As expected, in this case using a Gaussian noise model is beneficial, although DECI with a spline noise mode still performs strongly.

\begin{figure}
    \centering
    \includegraphics[width=.7\linewidth, trim = {0 3.5cm 0 0}, clip]{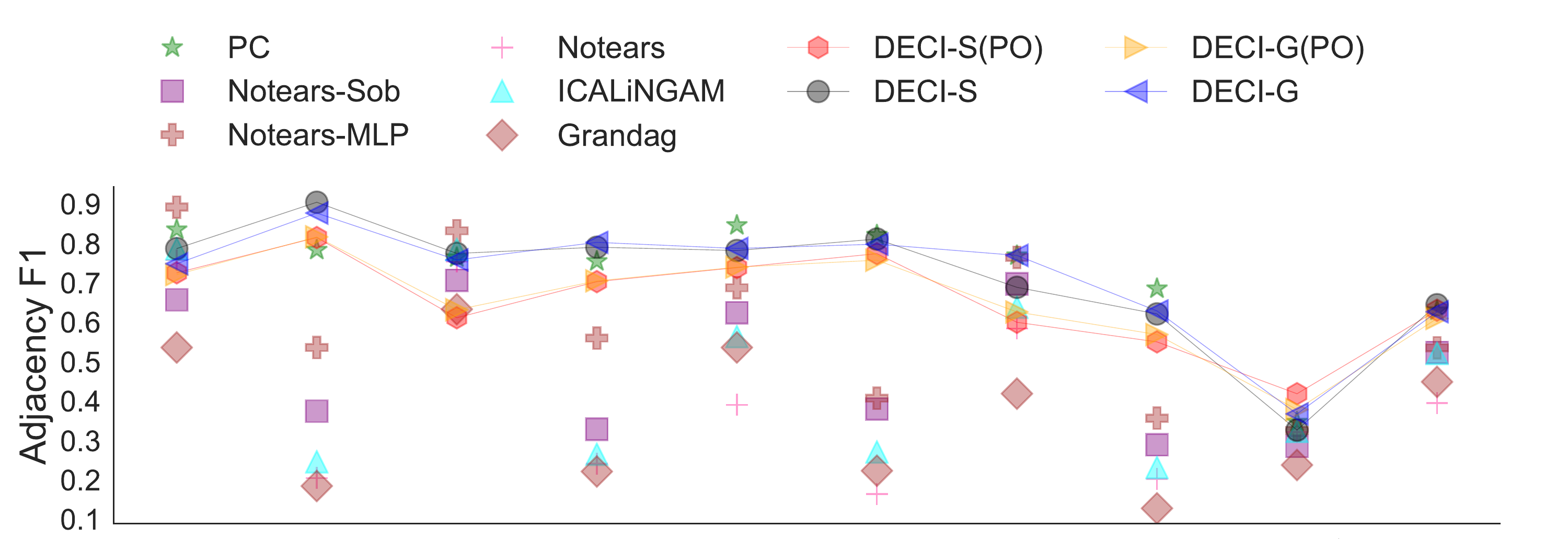}
    
    \includegraphics[width=.7\linewidth, trim = {0 3.5cm 0 2.95cm}, clip]{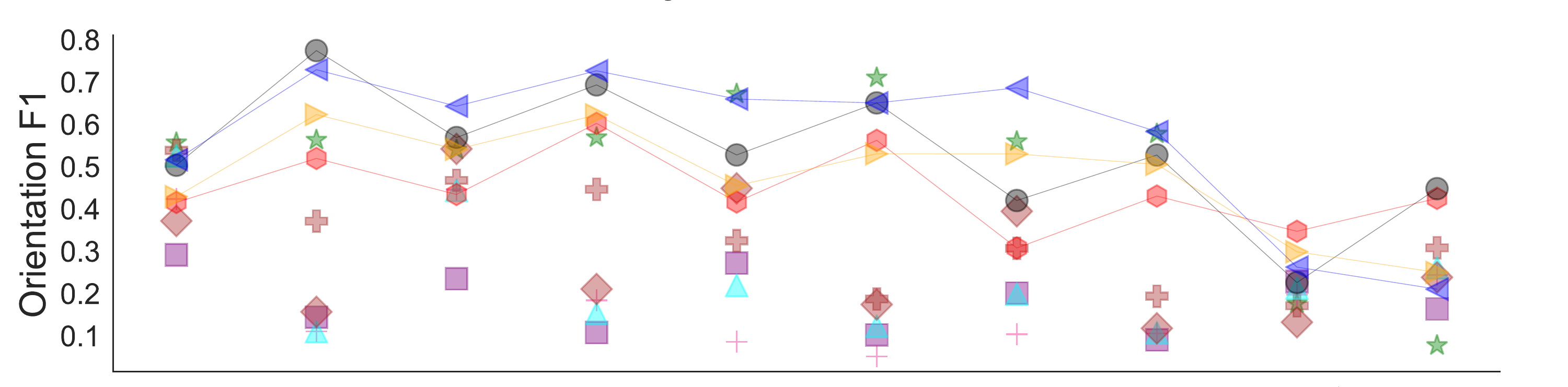}\
    
    \includegraphics[width=.7\linewidth, trim = {0 0 0 2.95cm}, clip]{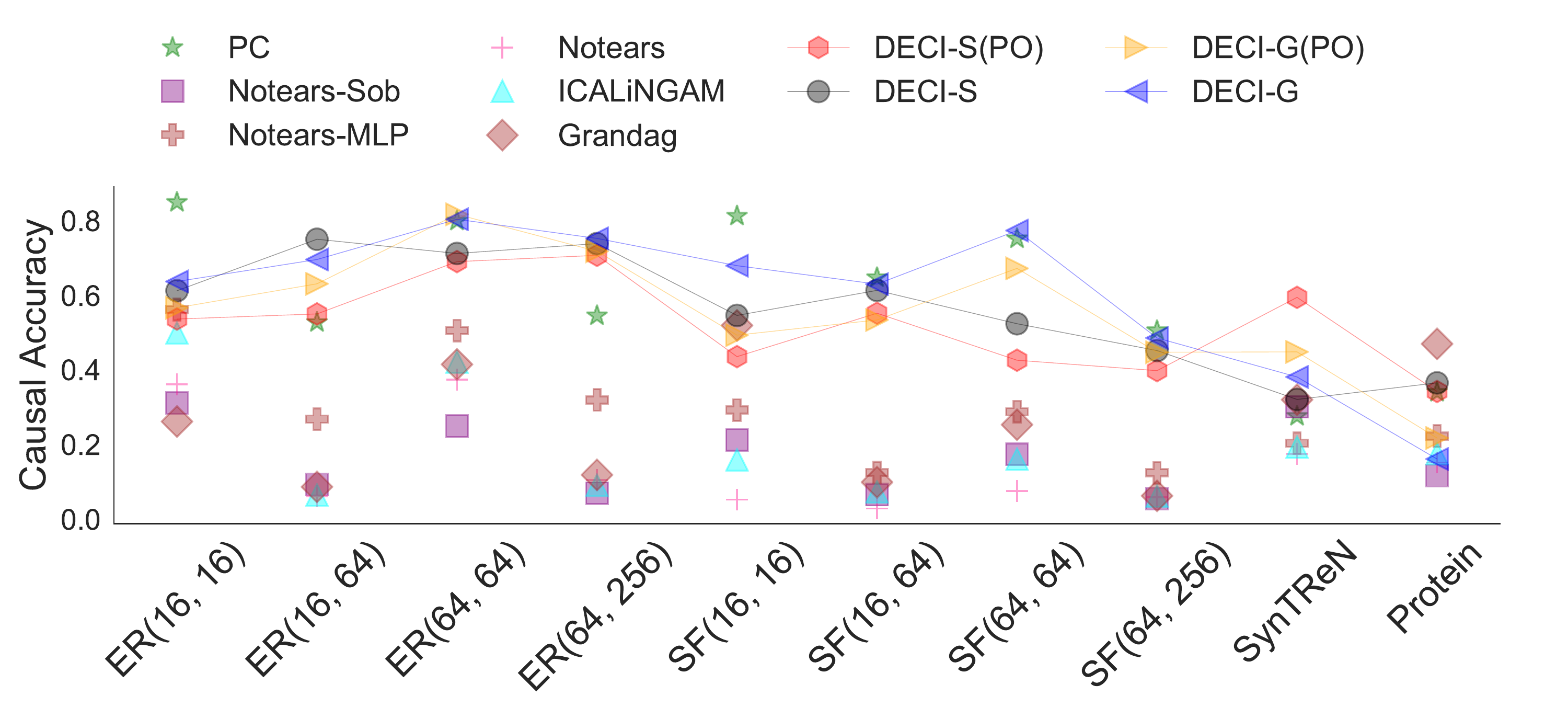}
    
    \caption{\textbf{DECI achieves better results than the baselines in all metrics shown.} The plots show the results for causal discovery for synthetic data generated using Gaussian noise. The legend ``DECI-G" and ``DECI-S" correspond to DECI using a Gaussian and spline noise model. Additionally, the ``(PO)'' corresponds to running DECI with 30\% of the training data missing completely at random. For readability, we highlight the DECI results by connecting them with soft lines. The figure shows mean results across five different random seeds.}
    \label{fig:cd_summary_gauss}
\end{figure}

\begin{figure}[p]
    \centering
    \includegraphics[width=0.95\linewidth]{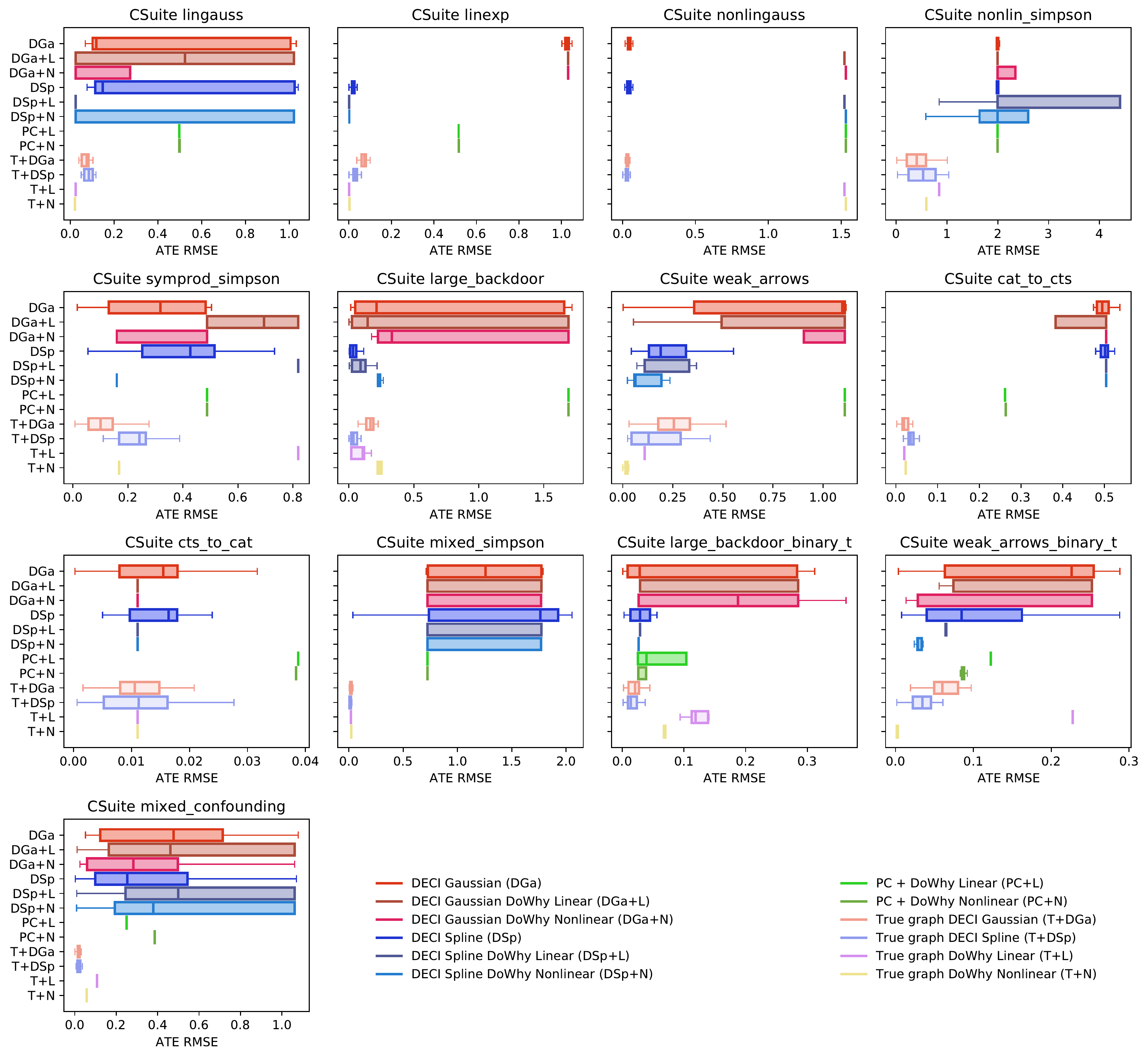}
    \caption{End-to-end ATE results on CSuite. }
    \label{fig:e2e_ate_rmse_csuite}
\end{figure}

\begin{figure}[p]
    \centering
    \includegraphics[width=0.9\linewidth]{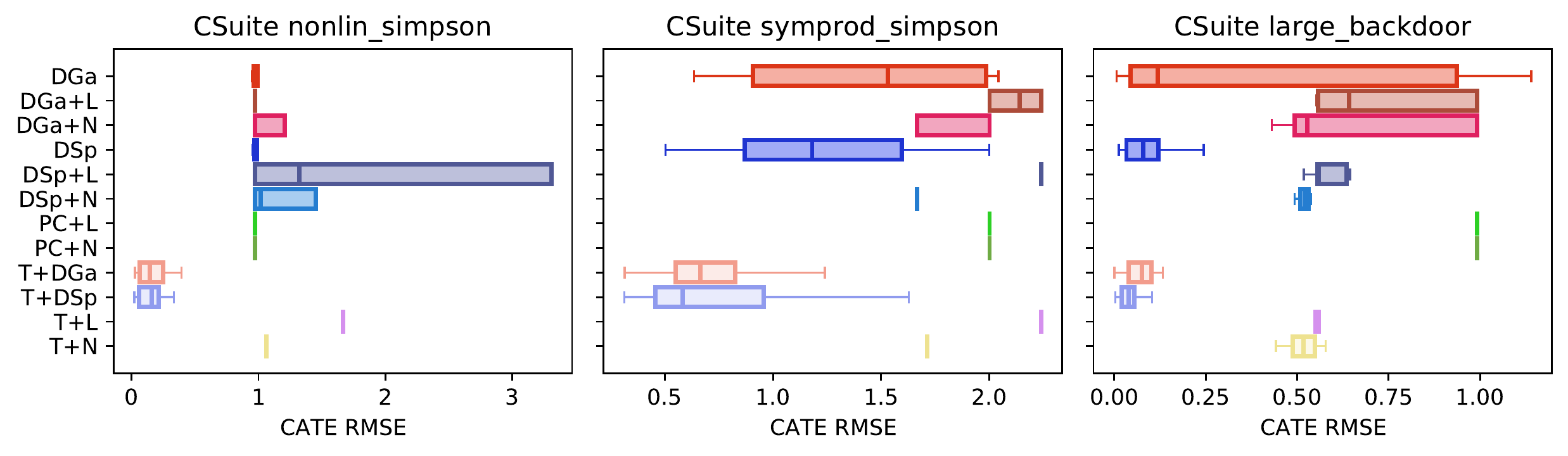}
    \caption{End-to-end CATE results on CSuite. Colours and acronyms as in Figure~\ref{fig:e2e_ate_rmse_csuite}.}
    \label{fig:e2e_cate_rmse_csuite}
\end{figure}

\begin{table}[t]
    \small
    \centering
    \resizebox{\columnwidth}{!}{%
    \begin{tabular}{lrrrrrrrrrrrr}
    \toprule
    Method & DGa & DGa+L & DGa+N & DSp & DSp+L & DSp+N & PC+L & PC+N & T+L & T+N & T+DGa & T+DSp \\
    Dataset                        &                 &                              &                                 &               &                            &                               &                 &                    &                         &                            &                            &                          \\
\midrule
ER(16, 16) - G                 &           1.280 &                        1.141 &                           1.129 &         1.648 &                      1.374 &                         1.393 &           1.491 &              1.475 &                   0.945 &                      0.936 &                      1.083 &                    1.332 \\
ER(16, 16) - S                 &           1.726 &                        1.776 &                           1.780 &         1.829 &                      1.776 &                         1.755 &           1.869 &              1.790 &                   1.755 &                      1.793 &                      1.643 &                    1.660 \\
ER(16, 64) - G                 &           1.699 &                        1.422 &                           1.334 &         1.501 &                      1.442 &                         1.202 &           1.335 &              1.440 &                   1.310 &                      1.369 &                      1.460 &                    1.644 \\
ER(16, 64) - S                 &           2.311 &                        2.465 &                           2.600 &         2.174 &                      2.452 &                         2.421 &           2.584 &              2.276 &                   2.742 &                      2.641 &                      2.510 &                    2.428 \\
ER(64, 64) - G                 &           1.208 &                        1.420 &                           1.190 &         1.287 &                      1.450 &                         1.397 &           1.325 &              1.273 &                   1.284 &                      1.250 &                      1.158 &                    1.124 \\
ER(64, 64) - S                 &           2.246 &                        1.626 &                           1.626 &         1.892 &                      2.325 &                         2.292 &           1.526 &              2.446 &                   2.442 &                      2.441 &                      2.490 &                    2.481 \\
SF(16, 16) - G                 &           1.156 &                        1.030 &                           1.409 &         1.699 &                      2.052 &                         1.343 &           1.574 &              1.233 &                   1.131 &                      1.078 &                      1.127 &                    1.375 \\
SF(16, 16) - S                 &           2.870 &                        1.805 &                           2.284 &         2.431 &                      2.363 &                         1.805 &           3.008 &              2.520 &                   2.518 &                      2.502 &                      2.424 &                    2.477 \\
SF(16, 64) - G                 &           1.702 &                          - &                             - &         1.539 &                        - &                         1.635 &           1.464 &              1.551 &                   1.510 &                      1.463 &                      1.559 &                    1.486 \\
SF(16, 64) - S                 &           3.594 &                          - &                             - &         4.139 &                        - &                           - &           3.877 &              4.145 &                   4.162 &                      4.106 &                      4.161 &                    3.861 \\
SF(64, 64) - G                 &           1.049 &                        0.998 &                           1.134 &         1.010 &                      1.035 &                         1.309 &           0.972 &              1.343 &                   1.006 &                        - &                      1.087 &                    1.288 \\
SF(64, 64) - S                 &           2.754 &                        3.239 &                           3.239 &         3.242 &                      3.239 &                         3.239 &           3.227 &              2.591 &                   2.815 &                        - &                      2.883 &                    3.034 \\
csuite\_cat\_to\_cts              &           0.495 &                        0.504 &                           0.504 &         0.501 &                      0.504 &                         0.504 &           0.262 &              0.264 &                   0.020 &                      0.023 &                      0.019 &                    0.036 \\
csuite\_cts\_to\_cat              &           0.015 &                        0.011 &                           0.011 &         0.016 &                      0.011 &                         0.011 &           0.039 &              0.038 &                   0.011 &                      0.011 &                      0.011 &                    0.011 \\
csuite\_large\_backdoor          &           0.213 &                        0.144 &                           0.331 &         0.031 &                      0.091 &                         0.232 &           1.690 &              1.690 &                   0.105 &                      0.241 &                      0.167 &                    0.035 \\
csuite\_large\_backdoor\_bt &           0.028 &                        0.029 &                           0.187 &         0.029 &                      0.029 &                         0.027 &           0.039 &              0.039 &                   0.119 &                      0.070 &                      0.021 &                    0.014 \\
csuite\_linexp                  &           1.029 &                        1.031 &                           1.031 &         0.022 &                      0.001 &                         0.002 &           0.516 &              0.517 &                   0.001 &                      0.003 &                      0.073 &                    0.028 \\
csuite\_lingauss                &           0.120 &                        0.523 &                           0.024 &         0.149 &                      0.025 &                         0.024 &           0.498 &              0.498 &                   0.025 &                      0.022 &                      0.076 &                    0.085 \\
csuite\_mixed\_confounding       &           0.477 &                        0.461 &                           0.282 &         0.254 &                      0.500 &                         0.380 &           0.250 &              0.387 &                   0.107 &                      0.057 &                      0.019 &                    0.018 \\
csuite\_mixed\_simpson           &           1.259 &                        0.723 &                           0.723 &         1.765 &                      1.772 &                         1.771 &           0.723 &              0.723 &                   0.017 &                      0.022 &                      0.014 &                    0.013 \\
csuite\_nonlin\_simpson          &           1.995 &                        1.994 &                           1.994 &         1.997 &                      1.994 &                         1.994 &           1.994 &              1.994 &                   0.848 &                      0.597 &                      0.404 &                    0.531 \\
csuite\_nonlingauss             &           0.042 &                        1.522 &                           1.532 &         0.043 &                      1.522 &                         1.532 &           1.532 &              1.532 &                   1.522 &                      1.532 &                      0.034 &                    0.034 \\
csuite\_symprod\_simpson         &           0.318 &                        0.695 &                           0.487 &         0.427 &                      0.819 &                         0.160 &           0.487 &              0.487 &                   0.819 &                      0.168 &                      0.101 &                    0.242 \\
csuite\_weak\_arrows             &           1.097 &                        1.108 &                           1.108 &         0.189 &                      0.110 &                         0.064 &           1.108 &              1.108 &                   0.109 &                      0.015 &                      0.255 &                    0.128 \\
csuite\_weak\_arrows\_bt    &           0.226 &                        0.252 &                           0.252 &         0.085 &                      0.065 &                         0.029 &           0.123 &              0.086 &                   0.228 &                      0.003 &                      0.060 &                    0.034 \\
IHDP                           &           0.187 &                        0.187 &                           0.187 &         0.090 &                      0.101 &                         0.116 &           0.187 &              0.187 &                   0.187 &                      0.187 &                      0.146 &                    0.087 \\
Twins                          &           0.030 &                        0.025 &                           0.025 &         0.022 &                      0.025 &                         0.025 &           0.068 &              0.025 &                   0.022 &                      0.042 &                      0.022 &                    0.060 \\
\bottomrule
\end{tabular}
}
    \vspace{2pt}
    \caption{Median ATE RMSE data underling our rank table. The median is taken across multiple seeds, with the number of seeds shown in Table~\ref{tab:seeds}. Standard deviations are also shown in Table~\ref{tab:stdev}. Missing values indicate that the method exceeded the computational budget---this typically occurred for larger graphs.}
    \label{tab:ate_rmse}
\end{table}

\begin{table}[t]
    \small
    \centering
    \resizebox{\columnwidth}{!}{%
    \begin{tabular}{lrrrrrrrrrrrr}
    \toprule
    Method & DGa & DGa+L & DGa+N & DSp & DSp+L & DSp+N & PC+L & PC+N & T+L & T+N & T+DGa & T+DSp \\
    Dataset                        &                 &                              &                                 &               &                            &                               &                 &                    &                         &                            &                            &                          \\
\midrule
ER(16, 16) - G                 &               5 &                            5 &                               5 &             5 &                          5 &                             5 &               5 &                  5 &                       5 &                          5 &                          5 &                        5 \\
ER(16, 16) - S                 &               5 &                            5 &                               5 &             5 &                          5 &                             5 &               5 &                  5 &                       5 &                          5 &                          5 &                        5 \\
ER(16, 64) - G                 &               5 &                            5 &                               5 &             5 &                          5 &                             5 &               5 &                  5 &                       5 &                          5 &                          5 &                        5 \\
ER(16, 64) - S                 &               5 &                            5 &                               4 &             5 &                          3 &                             5 &               5 &                  5 &                       5 &                          5 &                          5 &                        5 \\
ER(64, 64) - G                 &               5 &                            1 &                               1 &             5 &                          1 &                             1 &               4 &                  5 &                       5 &                          5 &                          5 &                        5 \\
ER(64, 64) - S                 &               5 &                            1 &                               1 &             5 &                          2 &                             2 &               1 &                  5 &                       5 &                          5 &                          5 &                        5 \\
SF(16, 16) - G                 &               5 &                            2 &                               3 &             5 &                          4 &                             5 &               4 &                  5 &                       5 &                          5 &                          5 &                        5 \\
SF(16, 16) - S                 &               5 &                            1 &                               2 &             5 &                          2 &                             1 &               3 &                  5 &                       5 &                          5 &                          5 &                        5 \\
SF(16, 64) - G                 &               5 &                            0 &                               0 &             5 &                          0 &                             1 &               4 &                  5 &                       5 &                          5 &                          5 &                        5 \\
SF(16, 64) - S                 &               5 &                            0 &                               0 &             5 &                          0 &                             0 &               4 &                  5 &                       5 &                          5 &                          5 &                        5 \\
SF(64, 64) - G                 &               5 &                            4 &                               3 &             5 &                          3 &                             2 &               3 &                  4 &                       5 &                          0 &                          5 &                        5 \\
SF(64, 64) - S                 &               5 &                            5 &                               5 &             5 &                          5 &                             5 &               3 &                  4 &                       5 &                          0 &                          5 &                        5 \\
csuite\_cat\_to\_cts              &              20 &                           20 &                              20 &            20 &                         20 &                            20 &              20 &                 20 &                      20 &                         20 &                         20 &                       20 \\
csuite\_cts\_to\_cat              &              20 &                           20 &                              20 &            20 &                         20 &                            20 &              20 &                 20 &                      20 &                         20 &                         20 &                       20 \\
csuite\_large\_backdoor          &              20 &                           20 &                              20 &            20 &                         20 &                            20 &              20 &                 20 &                      20 &                         20 &                         20 &                       20 \\
csuite\_large\_backdoor\_bt &              20 &                           20 &                              20 &            20 &                         20 &                            20 &              20 &                 20 &                      20 &                         20 &                         20 &                       20 \\
csuite\_linexp                  &              20 &                           20 &                              20 &            20 &                         20 &                            20 &              20 &                 20 &                      20 &                         20 &                         20 &                       20 \\
csuite\_lingauss                &              20 &                           20 &                              20 &            20 &                         20 &                            20 &              20 &                 20 &                      20 &                         20 &                         20 &                       20 \\
csuite\_mixed\_confounding       &              20 &                           20 &                              20 &            20 &                         20 &                            20 &              20 &                 20 &                      20 &                         20 &                         20 &                       20 \\
csuite\_mixed\_simpson           &              20 &                           20 &                              20 &            20 &                         20 &                            20 &              20 &                 20 &                      20 &                         20 &                         20 &                       20 \\
csuite\_nonlin\_simpson          &              20 &                           20 &                              20 &            20 &                         20 &                            20 &              20 &                 20 &                      20 &                         20 &                         20 &                       20 \\
csuite\_nonlingauss             &              20 &                           20 &                              20 &            20 &                         20 &                            20 &              20 &                 20 &                      20 &                         20 &                         20 &                       20 \\
csuite\_symprod\_simpson         &              20 &                           20 &                              20 &            20 &                         20 &                            20 &              20 &                 20 &                      20 &                         20 &                         20 &                       20 \\
csuite\_weak\_arrows             &              20 &                           20 &                              20 &            20 &                         20 &                            20 &              20 &                 20 &                      20 &                         20 &                         20 &                       20 \\
csuite\_weak\_arrows\_bt    &              20 &                           20 &                              20 &            20 &                         20 &                            20 &              20 &                 20 &                      20 &                         20 &                         20 &                       20 \\
IHDP                           &               5 &                            5 &                               5 &             5 &                          5 &                             5 &               5 &                  5 &                       5 &                          5 &                          5 &                        5 \\
Twins                          &               5 &                            5 &                               5 &             5 &                          5 &                             5 &               5 &                  5 &                       5 &                          5 &                          5 &                        5 \\
    \bottomrule
    \end{tabular}
    }
    \vspace{2pt}
    \caption{Number of seeds run when computing values in Table~\ref{tab:ate_rmse}. For ER/SF graphs, where fewer than 5 seeds were used, this indicates that some runs exceeded the computational budget.}
    \label{tab:seeds}
\end{table}

\begin{table}[t]
    \small
    \centering
    \resizebox{\columnwidth}{!}{%
    \begin{tabular}{lrrrrrrrrrrrr}
    \toprule
    Method & DGa & DGa+L & DGa+N & DSp & DSp+L & DSp+N & PC+L & PC+N & T+L & T+N & T+DGa & T+DSp \\
    Dataset                        &                 &                              &                                 &               &                            &                               &                 &                    &                         &                            &                            &                          \\
\midrule
csuite\_cat\_to\_cts              &         0.107 &                      0.210 &                         0.108 &       0.105 &                    0.194 &                       0.172 &           0.000 &              0.000 &                    0.011 &                  0.011 &                   0.000 &                      0.000 \\
csuite\_cts\_to\_cat              &         0.007 &                      0.000 &                         0.000 &       0.007 &                    0.000 &                       0.000 &           0.000 &              0.000 &                    0.009 &                  0.009 &                   0.000 &                      0.000 \\
csuite\_large\_backdoor          &         0.724 &                      0.737 &                         0.724 &       0.041 &                    0.062 &                       0.022 &           0.000 &              0.000 &                    0.046 &                  0.028 &                   0.055 &                      0.031 \\
csuite\_large\_backdoor\_bt &         0.134 &                      0.125 &                         0.124 &       0.081 &                    0.077 &                       0.056 &           0.033 &              0.006 &                    0.011 &                  0.010 &                   0.037 &                      0.004 \\
csuite\_linexp                  &         0.013 &                      0.000 &                         0.000 &       0.014 &                    0.000 &                       0.000 &           0.000 &              0.000 &                    0.015 &                  0.015 &                   0.000 &                      0.000 \\
csuite\_lingauss                &         0.435 &                      0.498 &                         0.432 &       0.454 &                    0.355 &                       0.489 &           0.000 &              0.001 &                    0.020 &                  0.022 &                   0.000 &                      0.000 \\
csuite\_mixed\_confounding       &         0.371 &                      0.389 &                         0.353 &       0.369 &                    0.418 &                       0.422 &           0.000 &              0.000 &                    0.010 &                  0.010 &                   0.000 &                      0.000 \\
csuite\_mixed\_simpson           &         0.524 &                      0.522 &                         0.507 &       0.585 &                    0.522 &                       0.514 &           0.000 &              0.000 &                    0.010 &                  0.008 &                   0.000 &                      0.000 \\
csuite\_nonlin\_simpson          &         1.071 &                      0.593 &                         0.976 &       1.080 &                    1.326 &                       1.409 &           0.000 &              0.000 &                    0.283 &                  0.304 &                   0.000 &                      0.000 \\
csuite\_nonlingauss             &         0.014 &                      0.000 &                         0.000 &       0.016 &                    0.000 &                       0.000 &           0.000 &              0.000 &                    0.013 &                  0.016 &                   0.000 &                      0.000 \\
csuite\_symprod\_simpson         &         0.187 &                      0.162 &                         0.162 &       0.180 &                    0.000 &                       0.128 &           0.000 &              0.000 &                    0.090 &                  0.124 &                   0.000 &                      0.000 \\
csuite\_weak\_arrows             &         0.433 &                      0.401 &                         0.426 &       0.256 &                    0.118 &                       0.075 &           0.000 &              0.000 &                    0.160 &                  0.131 &                   0.000 &                      0.008 \\
csuite\_weak\_arrows\_bt    &         0.104 &                      0.118 &                         0.113 &       0.093 &                    0.067 &                       0.067 &           0.000 &              0.002 &                    0.020 &                  0.017 &                   0.000 &                      0.001 \\
IHDP                           &         0.021 &                      0.000 &                         0.062 &       0.013 &                    0.048 &                       0.037 &           0.024 &              0.024 &                    0.014 &                  0.022 &                   0.000 &                      0.000 \\
Twins                          &         0.018 &                      0.000 &                         0.000 &       0.004 &                    0.022 &                       0.019 &           0.023 &              0.024 &                    0.003 &                  0.009 &                   0.000 &                      0.000 \\
    \bottomrule
    \end{tabular}
    }
    \vspace{2pt}
    \caption{Standard deviations for ATE RMSE results.
    }
    \label{tab:stdev}
\end{table}

\subsection{Summary of CSuite Results}

Comprehensive results on CSuite ATE and CATE performance are shown in figures~\ref{fig:e2e_ate_rmse_csuite} and~\ref{fig:e2e_cate_rmse_csuite}.
We first provide a summary of results here and then go into per-dataset analysis in the following subsection. 

We find DECI to perform consistently well in our 2 node datasets. It learns a uniform posterior over graphs in the non-identifiable setting, it fits non-linear functions well and it is robust to heavy tailed noise when employing the spline noise model. We find linear and non-linear DML inference to also perform acceptably, with the exception of the heavy tailed noise case, where the methods overfit to outliers and thus estimate ATE poorly.
    
    On the larger (4 and 12 node) datasets, when the true graph is available, DECI provides ATE estimates competitive with the well-established non-linear DML method. Notably, DECI outperforms backdoor adjustment methods when the number of possible adjustment sets is large. Choosing the optimal adjustment set is an np-hard problem and the most common approach is to simply choose the largest one. This leads to doWhy suffering from variance. DECI's simulation-based approach avoids having to choose an adjustment set. On the other hand, for densely connected graphs where the strength of the connection between nodes is low, DECI struggles to capture the funcitonal relationships in the data and DML is most competitive. For CATE estimation DECI provides superior performance in all datasets and is able to completely solve all tasks but one.

  When the graph is learnt from the data, the non-linear nature of our (4 and 12 node) datasets together with their heavy tailed noise make the discovery problem very challenging. We find that the PC algorithm provides very poor results or fails to find any causal DAGs compatible with the data when working with these datasets. We find both DECI to provide more acceptable performance with the DECI-spline variant producing more reliable results.
In this learnt graph setting, causal inference performance deteriorates sharply as a consequence of imperfect causal discovery.  
However, our findings in terms of relative performance among inference methods stay the same.

\subsection{Discussion of Continuous CSuite Results}
\begin{enumerate}
    \item \textbf{lingauss}: When the true graph is available, all our causal inference methods are able to solve this problem. However, when the graph needs to be identified from the data, causal discovery accuracy is around 50\%. DECI discovery converges to a posterior with half of its mass on the right distribution resulting in DECI inference methods showing the lowest error.
    \item \textbf{linexp}: The non-Gaussian noise causes difficulties for DECI-Gaussian, which identifies the wrong orientation in a majority of cases. As a result, inference algorithm yield poor results. 
    Surprisingly, the PC algorithm is also unable to identify the causal graph, leading to overall poor inference performance. 
    With the spline noise model, DECI successfully identifies the causal graph, allowing for all inference algorithms to solve the problem.
    \item \textbf{nonlingauss}: The non-linear relationship between variables leads all DECI discovery runs to successfully recover the edge direction for this dataset while PC consistently identifies the wrong edge direction. As expected, linear ATE estimation performs poorly on this task. However, we find DoWhy non-linear to not fare much better, likely this is because DML still assumes a linear relationship between treatment and target.
    DECI solves the task successfully.
    
    \item \textbf{nonlin\_simpson}: Even when the true graph is available, none of our inference methods are able to recover the true ATE on this more difficult task. We observe non-linear methods (DECI and DoWhy-nonlinear) to perform similarly to each other and more strongly than the simple linear adjustment. For the CATE task, the true value is close to 0. This is correctly identified by both DECI-Gaussian and DECI-Spline. Interestingly, we find both linear and non-linear DoWhy variants to overestimate the causal effect when using the backdoor criterion. We attribute this to DECI solving a lower dimensional problem when estimating CATE. While DECI simply regresses the conditioning variable onto the effect variable. The backdoor adjustment employed by DoWhy requires regression from the joint space of conditioning variables and confounders onto the effect variables. The latter procedure involves estimating the relative strength of confounders and conditioning variables, which is a more challenging task.
    
    This dataset provides a challenging causal discovery task. DECI identifies the correct edges with probability ~0.9. It capacity to recover the edge orientation is slightly worse ~0.65. This imperfect causal discovery leads to poor inference for all methods. (potentially because they get Simpson's paradox the wrong way around).
    \item \textbf{symprod\_simpson}: Even with access to the true graph, no inference method is able to solve this problem. However, we find non-linear methods to clearly outperform linear adjustment for both CATE and ATE estimation. Among non-linear methods, performance is similar for ATE estimation, with DECI-Gaussian performing slightly better than DoWhy-nonlinear and DECI-Spline slightly worse. However, when estimating CATE, DECI inference present an error twice as low as nonlinear DoWhy. Again, we attribute this to the backdoor adjustment employed by DoWhy being a more challenging inference task than the 1d regression on simulated data employed by DECI.
    
    In terms of causal discovery, results are similar to nonlin-simpson with PC failing completely and DECI obtaining an adjacency score of ~0.92 and orientation of ~0.7. The imperfect graph knowledge hurts causal inference. Again we see the non-linear backdoor adjustment to perform similarly to DECI for ATE estimation while DECI shows decisively stronger performance when estimating CATE. As expected, the linear adjustment method fares poorly in this strongly non-linear setting.
    
    \item \textbf{weak\_arrows}: When the true causal DAG is available we find that both DoWhy methods solve this ATE problem while both DECI methods predict slightly suboptimal ATE values. 
    
    In terms of causal discovery, DECI clearly outperforms PC with the spline noise model again proving more reliable and leading to better ATE estimates. Although no methods are able to solve the task, we find that non-linear DoWhy with the DECI-spline graphs performs best. We hypothesize that the amortised function structure employed by DECI suffers in very densely connected graphs with weak edges, like is the case here.
    
    \item \textbf{large\_backdoor}: With access to the true graph, DECI methods outperform both Dowhy variants for both ATE and CATE estimation. DECI-spline performs best and is able to solve both problems. When faced with many confounders, adjustment procedures suffer from large variance. As a result, 
    despite the non-linearity of the functional relationships at play, the simpler linear backdoor adjustment outperforms the non-linear DML approach.  On the other hand, DECI's simulation based approach is not disadvantaged in this setting.
    
    Following the trend of the previous datasets, PC performs poorly in terms of causal discovery, biasing downstream inference methods which perform poorly in terms of ATE and CATE estimation. DECI discovery is more reliable, an effect most noticeable when using the spline noise models. With the DECI-Spline posterior over graphs, both DECI-spline and linear DoWhy are able to solve the ATE problem and DECI-spline is the only method capable of solving the CATE task. For both tasks and noise models DECI outperforms non-linear DoWhy, again showing its invariance to the size of potential adjustment set. 

\end{enumerate}

\subsection{Synthetic Graph Experiments}

\begin{figure}[p]
    \centering
    \includegraphics[width=0.95\linewidth]{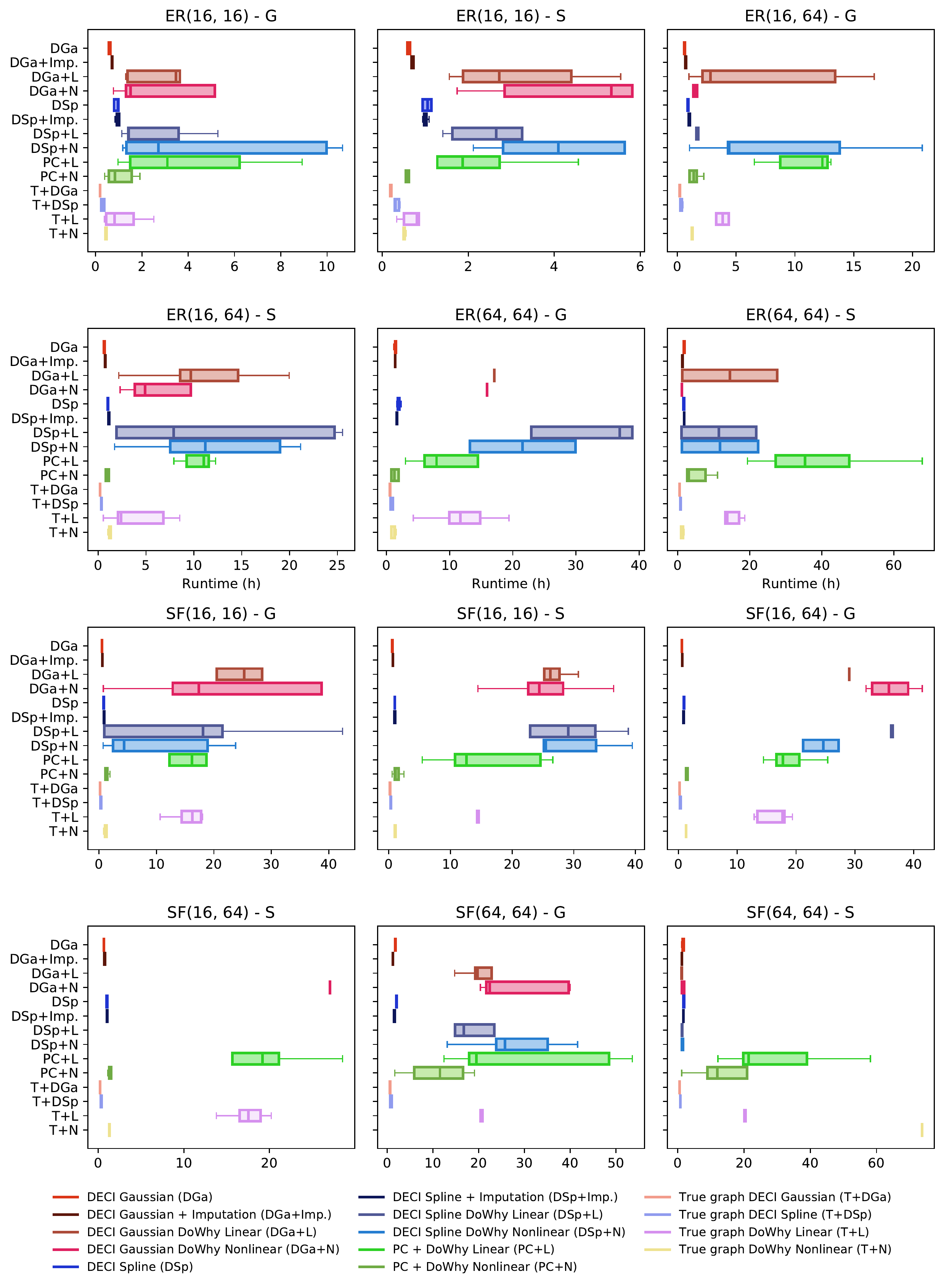}
    \caption{Runtime of End-to-end ATE estimation methods on synthetic graphs. }
    \label{fig:e2e_runtime_rmse_ersf}
\end{figure}

We test the performance of DECI on ATE estimation with random graphs as described in \cref{sec:CD_exp}. For each graph, we randomly generate interventional data for up to five random interventions. We chose the effect variable as the last variable in the causal order that has not yet been used for data generation. For each effect variable we chose the intervention by randomly traversing the graph up to three edges away from the effect variable.

\Cref{tab:ate_rmse} shows the performance of the ATE estimation of DECI and all baselines on the synthetic graph data. We only show results for methods that have a runtime of less than one day. 
\Cref{fig:e2e_runtime_rmse_ersf} shows the runtimes for the different methods. DECI has consistently the lowest runtime and scales best to larger graphs. While the runtime of DECI stays approximately constant for various graphs, the runtime of the ATE estimation baselines increases with more complex graphs.
In general, the methods using the true graph outperform the methods that also perform causal discovery. Further, no method strongly outperforms all other methods with DECI being a strong competitor to the already established DML methods. Lastly, we can see that DECI is capable of performing causal discovery, data imputation and ATE estimation in an end-to-end fashion without degrading performance.

\subsection{Learning in Non-identifiable Settings with the Help of Graph Priors}

We investigate the utility of prior knowledge over causal graphs for causal discovery and end2end inference in non-identifiable and difficult to identify settings. Specifically, we generate 2 datasets composed of 2000 training examples each. The first is composed of only linear relationships between variables and Gaussian additive noise, making the causal graph non-identifiable. The second dataset also uses linear functions but has a mix of exponential and Tanh-Gaussian noise. Although identifiable, discovery in this latter setting is challenging.

We introduce prior knowledge about graph sparseness through
the weighted adjacency matrix $W_{0} \in [0,1]^{D \times D}$, with zero entries encouraging sparser graphs. The resulting informed DECI prior is
\begin{equation*}
    p(G) \propto \exp \left(-\lambda_s \Vert G - W_{0}\Vert_F^{2} - \rho \, h(G)^2 - \alpha \, h(G)\right), 
\end{equation*}
with the scalar $\lambda_{s}$ regulating the strength of the prior beliefs encoded in $W_{0}$.

We compare DECI inference with access to the true graph to end2end DECI inference. In the latter case we consider a PC prior, which has as its mean the CP-DAG provided by PC. We consider different prior strengths, i.e. the value of the entries of $W_{0}$, between 0 and 1. We also experiment with introducing the true-graph as a prior of this form, yielding what we refer to as the ``informed prior''.

In the non-identifiable case, we find both DECI (prior strength 0) and PC discovery to provide incorrect graphs. Interestingly, providing the PC CPDAG as a prior for DECI can yield large gains in terms of causal discovery due to a variance reduction effect. These gains do not translate to better ATE estimation, where performance is not improved over the uninformative prior. Providing knowledge of the true graph does help causal inference, with a more confident prior yielding better results.

In the difficult identifiable case, the PC prior does provide gains to DECI. We find the optimal prior strength to be 0.5: a balanced combination of PC and DECI discovery is most reliable, while using exclusively one of the two algorithms yields worse results. In this identifiable setting informed DECI discovery is able to obtain perfect ATE estimation performance with a prior strength as low as 0.2.

\begin{figure*}[htb]
    \centering
    \includegraphics[width=\linewidth]{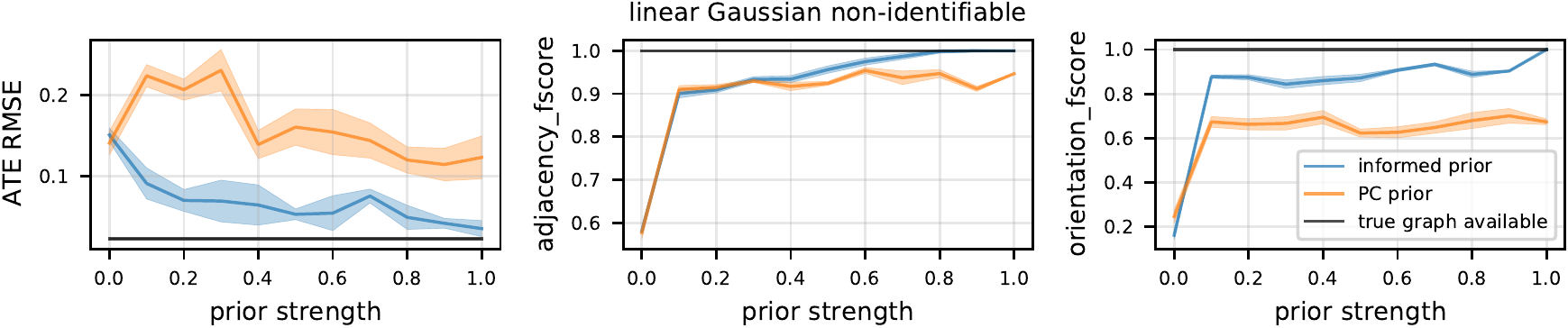}
    \caption{Causal discovery and inference results obtained on a 9 node linear Gaussian dataset, where the graph is non-identifiable without prior knowledge. We perform DECI inference with the true graph and DECI end2end inference with different priors. Informed prior refers to using a smoothed version of the true graph as the prior. PC prior refers to using the CP-DAG outputted by PC as the prior mean $W_{0}$. The prior strength indicates how much prior mass is placed on the mean prior graph $W_{0}$ and how much is spread across all other DAGs.}
    \label{fig:priors_non_identifiable}
\end{figure*}

\begin{figure*}[htb]
    \centering
    \includegraphics[width=\linewidth]{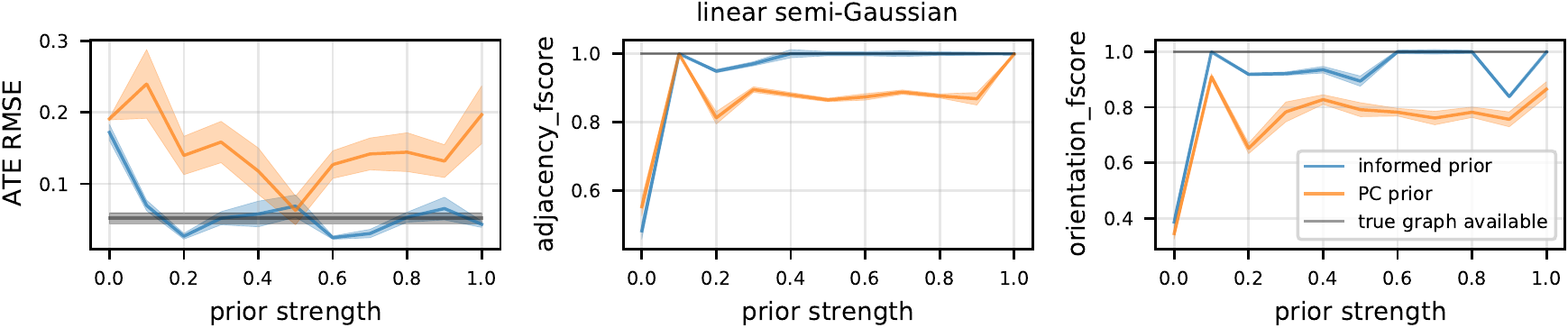}
    \caption{Causal discovery and inference results obtained on a 9 node linear non-Gaussian dataset, difficult to identify without prior knowledge. We perform DECI inference with the true graph and DECI end2end inference with different priors. Informed prior refers to using a smoothed version of the true graph as the prior. PC prior refers to using the CP-DAG outputted by PC as the prior mean $W_{0}$. The prior strength indicates how much prior mass is placed on the mean prior graph $W_{0}$ and how much is spread across all other DAGs.}
    \label{fig:priors_non_identifiable2}
\end{figure*}

\end{document}